%% file: main.tex
\documentclass[accepted]{uai2026}

\usepackage[utf8]{inputenc}
\usepackage[T1]{fontenc}
\usepackage{microtype}
\usepackage[english]{babel}

\usepackage{natbib}
\usepackage{authblk}

\usepackage{algorithm}
\usepackage{algorithmic}

\usepackage{graphicx}
\usepackage{subcaption}
\usepackage{booktabs}
\usepackage{multirow}
\usepackage[usenames,dvipsnames,svgnames,table]{xcolor}

\usepackage{amsmath}
\usepackage{amssymb}
\usepackage{mathtools}
\usepackage{amsthm}
\usepackage{amsfonts}
\usepackage{bm}

\usepackage{enumitem}
\setlist{nosep}

\usepackage[most]{tcolorbox}
\usepackage{tikz}
\usepackage{fontawesome5}

\definecolor{takeawaybg}{HTML}{F2F7F7}
\definecolor{takeawaystroke}{HTML}{468985}
\definecolor{theorbg}{gray}{0.96}
\definecolor{theorframe}{gray}{0.40}

\newtcolorbox{contribbox}{
  enhanced,
  colback=green!5,
  colframe=green!40!black,
  boxrule=0pt,
  arc=0mm,
  left=2mm, right=2mm, top=1.5mm, bottom=1.5mm,
  borderline west={3pt}{0pt}{green!60!black},
}

\usepackage{hyperref}

\definecolor{uaiblue}{RGB}{0, 0, 120}

\hypersetup{
  colorlinks=true,
  linkcolor=uaiblue,
  citecolor=uaiblue,
  urlcolor=uaiblue
}

\usepackage[capitalize,noabbrev]{cleveref}

\theoremstyle{plain}
\newtheorem{theorem}{Theorem}[section]
\newtheorem{proposition}[theorem]{Proposition}

\newtheorem{corollary}[theorem]{Corollary}

\theoremstyle{definition}
\newtheorem{definition}[theorem]{Definition}
\newtheorem{assumption}[theorem]{Assumption}

\theoremstyle{remark}
\newtheorem{remark}[theorem]{Remark}

\tcolorboxenvironment{theorem}{
  enhanced,
  colback=theorbg,
  colframe=theorframe,
  boxrule=0pt,
  arc=0pt,
  left=2mm,
  right=2mm,
  top=1.5mm,
  bottom=1.5mm,
  borderline west={2pt}{0pt}{theorframe},
}

\tcolorboxenvironment{remark}{
  enhanced,
  colback=takeawaybg,
  colframe=takeawaybg,
  boxrule=0pt,
  arc=0pt,
  left=4mm,
  right=2mm,
  top=2mm,
  bottom=2mm,
  borderline west={3pt}{0pt}{takeawaystroke},
  overlay={
    \node[
      anchor=center,
      fill=white,
      draw=takeawaystroke,
      line width=1pt,
      circle,
      inner sep=2pt,
      yshift=-8pt,
      xshift=1.5pt
    ] at (frame.north west) {\color{takeawaystroke}\footnotesize\faLightbulb};
  }
}

\newcommand{\R}{\mathbb{R}}
\newcommand{\E}{\mathbb{E}}

\newcommand{\ip}[2]{\left\langle #1 ,\, #2 \right\rangle}

\DeclareMathOperator{\conv}{conv}

\DeclareMathOperator*{\Argmin}{Arg\,min}
\DeclareMathOperator*{\argmin}{arg\,min}

\DeclareMathOperator*{\argmax}{arg\,max}

\title{Which Directions Matter? Sparse Design for Affine Robust Optimization}

\author[1]{Pedro Chumpitaz-Flores}
\author[1]{My Duong}
\author[1]{Juan S. Borrero}
\author[1]{Kaixun Hua}

\affil[1]{%
  \textit{University of South Florida}\\
  Tampa, FL, USA%
}

\begin{document}

\maketitle

\input{content/0_Content}
\bibliography{custom}

\clearpage
\appendix
\input{content/0_appendix}

\end{document}

%% file: content/0_Content.tex
\begin{abstract}
Robust machine learning and optimization rely on the uncertainty model choice. We investigate which uncertainty directions a model must cover when defined by a finite dictionary and a budget constraint. Selecting a subset forms an atomic uncertainty set with a closed form support function, yielding tractable robust programs for affine objectives. We propose a data-driven selection rule based on a coverage objective over evaluation directions, including gradients, adversarial perturbations, or shifts observed on held out data. We prove this objective is monotone and submodular, supporting a greedy method with a $(1-1/e)$ approximation guarantee and a matching hardness barrier. We also provide a certificate bounding the loss from the selected subset and a radius calibration rule with out-of-sample control.
\end{abstract}

\input{content/0_Intro}
\input{content/1_Method}
\input{content/2_Experiments}

\input{content/3_Conclusion}

%% file: content/0_Intro.tex
\section{Introduction}

Directional uncertainty often involves large families of perturbation directions, yet only a small subset influences the robust optimum. This paper investigates: Given a finite dictionary $\mathcal D=\{d_i\}_{i=1}^N$ and a budget $B$, which directions recover the robustness of the full model?

We address this through an atomic directional uncertainty family $\mathcal{U}_S(r)$ indexed by subsets $S\subseteq[N]$, a submodular support coverage surrogate for budget-constrained selection, and certificates linking directional coverage to the robust value gap. We further provide finite-sample radius calibration for out-of-sample feasibility.

\begin{figure}[h]
    \centering
    \includegraphics[width=1\linewidth]{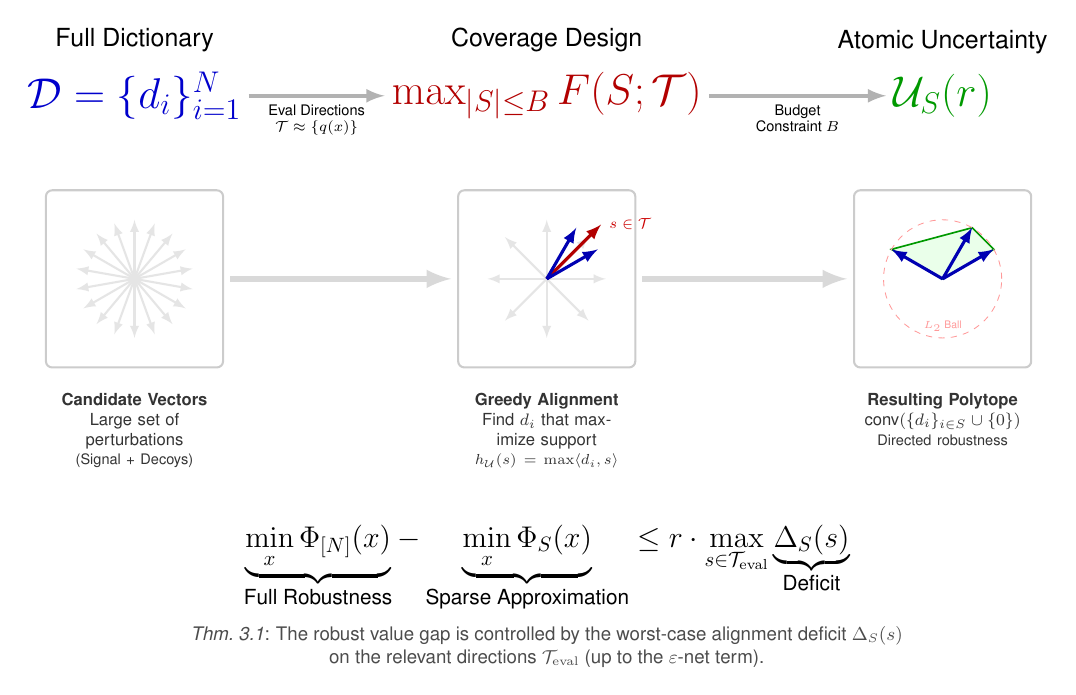}
    \vspace{-0.7cm}
    \caption{\textbf{Sparse directional design via greedy support alignment.} From a dictionary $\mathcal{D}$ of candidate perturbations containing signal and noise (left), the strategy (center) greedily identifies atoms $d_i$ maximizing support alignment with evaluation directions $\mathcal{T}$, such as gradients. This defines the atomic uncertainty set $\mathcal{U}_S(r)$ as a sparse polytope (right). Section~\ref{sec:theoretical-guarantees} establishes selection submodularity and optimality, while the performance certificate (Theorem~\ref{thm:enet-bridge}) ensures minimizing the alignment deficit $\Delta_S$ on $\mathcal{T}$ recovers the robust value.}
    \label{fig:UAI2026_Uncertanty__Fig1_}
\end{figure}

Our approach builds on robust optimization via uncertainty set design and support function reformulations \citep{BenTal1998, ElGhaoui1997, Bertsimas2004, Bertsimas2011}, motivated by certified learning under structured perturbations \citep{Szegedy2014, Goodfellow2015, Madry2018, Wong2018, Raghunathan2018, Cohen2019, Engstrom2019}. Algorithmically, we employ monotone submodular maximization with cardinality constraints \citep{Nemhauser1978, Krause2014} and integrate a calibration layer inspired by conformal prediction coverage guarantees \citep{Vovk2005, Shafer2008, Lei2018, Romano2019}.
The dictionary $\mathcal D$ comprises application-specific candidates, including gradients, adversarial directions at training points, domain shifts, or stress directions used in robust optimization. Our experiments instantiate $\mathcal D$ using synthetic shifts and classifier feature directions to study the micro-budget regime $B \ll N$.

\textbf{Contributions.}
We present a framework for designing budget-constrained directional uncertainty sets.
\begin{itemize}[leftmargin=*, topsep=2pt, itemsep=2pt, parsep=0pt]
\item \textbf{Conceptually:} We frame directional robustness as a subset design problem over a finite dictionary, where a small direction set recovers the full model performance.
\item \textbf{Technically:} We introduce an atomic uncertainty family with a closed-form support function, deriving certificates that relate directional coverage to the robust value gap and providing finite-sample radius calibration for out-of-sample feasibility.
\item \textbf{Empirically:} We demonstrate that greedy selection based on support coverage outperforms random baselines in the micro-budget regime, recovering full model behavior with a minimal fraction of directions.
\end{itemize}

%% file: content/1_Method.tex
\section{Problem Setup and Performance Gap}

\subsection{Atomic Sets and Dictionaries}
\label{sec:formulation}

We consider a finite dictionary $\mathcal{D}=\{d_1,\dots,d_N\}\subset\R^p$ augmented with the null atom $d_0:=\bm{0}$, and denote the index set by $[N]_0:=\{0,1,\dots,N\}$. For any subset $S\subseteq[N]$, let $S_0:=S\cup\{0\}$. We define the atomic set dependent on a radius $r>0$ as
\begin{equation}\label{eq:US}
U_S(r)\ :=\ \Big\{u=\sum_{i\in S}\alpha_i d_i:\ \alpha_i\ge 0,\ \sum_{i\in S}\alpha_i\le r\Big\},
\end{equation}
which implies that the support function takes the closed form
\begin{equation}\label{eq:support}
h_{U_S}(s)\ =\ \sup_{u\in U_S(r)}\ip{u}{s}\ =\ r\,\max_{i\in S_0}\ip{d_i}{s}.
\end{equation}
Although \eqref{eq:US} relies on non-negative coefficients $\alpha_i \ge 0$, symmetric uncertainty sets such as $\ell_1$-balls are accommodated by including antipodal pairs $\{d, -d\}$ in the dictionary $\mathcal{D}$.

\subsection{Affine Robust Optimization}

We address a baseline robust optimization problem defined over a nonempty and compact domain $X\subseteq\R^n$.

\begin{assumption}[Affine structure in $u$ and regularity]\label{asp:lin-u}
There exist continuous functions $b:X\to\R$ and $q:X\to\R^p$ such that the objective satisfies $f(x,u)=b(x)+\ip{u}{q(x)}$ for all pairs $(x,u)$. Accordingly, the robust objective is
\begin{equation}\label{eq:Phi_definition}
\begin{split}
    \Phi_S(x)\ &:=\ \sup_{u\in U_S(r)} f(x,u) \\
    &=\ b(x)+h_{U_S}\big(q(x)\big) \\
    &=\ b(x)+r\max_{i\in S_0}\ip{d_i}{q(x)}.
\end{split}
\end{equation}
Under the compactness of $X$, minimizers of $\Phi_S$ are guaranteed to exist.
\end{assumption}

\begin{remark}[Optional convexity for tractability]
If $X$ is convex, $b$ is convex, and $q$ is affine, then the function $x\mapsto \Phi_S(x)$ is convex. This structure covers standard linear and second order cone programming instances by linearizing the maximum in the objective.
\end{remark}

Given an out-of-sample evaluation metric $\mathrm{Risk}(x)$, such as the hold out violation rate or the conditional value at risk of the regret, the subset design problem seeks to minimize
\begin{equation}\label{eq:diseno}
\min_{\substack{S\subseteq[N]\\ |S|\le B}}\ \mathrm{Risk}\!\big(x_S^\star\big)\qquad
\text{where}\quad x_S^\star\in\Argmin_{x\in X}\ \Phi_S(x).
\end{equation}

\subsection{Performance Gap and Certificates}

To quantify the loss of robustness, we define the full model corresponding to $S=[N]$ with objective $\Phi_{[N]}(x)=b(x)+r\max_{i\in [N]_0}\ip{d_i}{q(x)}$. The following result characterizes the difference between the full and restricted objectives.


\begin{proposition}[Pointwise value gap]\label{prop:gap-ptwise}
For any $x\in X$ and $S\subseteq[N]$, the value gap is given by
\begin{equation}\label{eq:gap-ptwise}
\begin{split}
\Phi_{[N]}(x) &- \Phi_S(x) \\
&= r\Big(\max_{i\in [N]_0}\ip{d_i}{q(x)}-\max_{i\in S_0}\ip{d_i}{q(x)}\Big) \\
&\ge 0.
\end{split}
\end{equation}
In particular, if $x_S^\star$ is a minimizer of $\Phi_S$, then
\begin{equation}\label{eq:gap-optvals}
\begin{split}
\min_{x}\Phi_{[N]}(x) &- \min_{x}\Phi_S(x) \\
&\le \Phi_{[N]}(x_S^\star)-\Phi_S(x_S^\star) \\
&= r\,\Delta_S\big(q(x_S^\star)\big),
\end{split}
\end{equation}
where the deficit is defined as $\Delta_S(s):=\max_{i\in [N]_0}\ip{d_i}{s}-\max_{i\in S_0}\ip{d_i}{s}$.
\end{proposition}

\begin{corollary}[Sufficient condition for equality of values]\label{cor:igualdad-valores}
If there exists a solution $x_S^\star\in\Argmin\Phi_S$ such that the deficit satisfies $\Delta_S\big(q(x_S^\star)\big)=0$, then the optimal values coincide:
\[
\min_x \Phi_{[N]}(x)\ =\ \min_x \Phi_S(x).
\]
We note that this does not imply coincidence of the minimizers except under additional properties such as uniqueness or strict convexity.
\end{corollary}

\begin{remark}[Critical directions]
We refer to the vectors $s_t:=q(x_t)$ as revealed directions, where $x_t$ is a solution of $\min_x \Phi_{S_t}(x)$ obtained during the algorithm described in Section \ref{sec:alg}. Monitoring the value $\Delta_{S_t}(s_t)$ allows us to certify the condition of Corollary \ref{cor:igualdad-valores} at run time.
\end{remark}

\section{Theoretical and Statistical Guarantees}
\label{sec:theoretical-guarantees}

\subsection{Deterministic Error Bounds}
\label{sec:enet-bridge}

\paragraph{Metric definitions.}
To establish a deterministic global bound on the value gap, we first recall the definition of the alignment deficit $\Delta_S(s):=\max_{i\in [N]_0}\langle d_i,s\rangle-\max_{i\in S_0}\langle d_i,s\rangle\ge 0$, where $S_0=S\cup\{0\}$ and $[N]_0=\{0,1,\dots,N\}$ with $d_0=\bm 0$. We introduce the pseudometric induced by the dictionary, defined as
\begin{equation}\label{eq:dD}
d_{\mathcal D}(s,s')\ :=\ \max_{i\in [N]_0}\,|\langle d_i,\, s-s'\rangle|.
\end{equation}
Based on this metric, a finite set $\mathcal T\subset\R^p$ constitutes an $\varepsilon$-net of $q(X)$ if, for every $x\in X$, there exists an element $s\in\mathcal T$ such that $d_{\mathcal D}\big(q(x),s\big)\le \varepsilon$. This geometric structure bridges the gap between discrete directional coverage and global robustness.

\begin{theorem}[$\varepsilon$-net bridge theorem]\label{thm:enet-bridge}
Under Assumption~\ref{asp:lin-u} and compactness of $X$, for any subset $S\subseteq[N]$ and any $\varepsilon$-net $\mathcal T$ of $q(X)$ in $d_{\mathcal D}$, the value gap satisfies
\[
\min_{x}\Phi_{[N]}(x)\ -\ \min_{x}\Phi_{S}(x)
\ \le\ r\Big(\max_{s\in\mathcal T}\Delta_S(s)\ +\ 2\varepsilon\Big).
\]
\end{theorem}

\begin{proof}[\textit{Sketch (Full proof in App.~\ref{app:theom3_1})}]
For each $x$, the support function $h_{U}(q(x))$ is $r$-Lipschitz in $d_{\mathcal D}$, meaning $|h_{U}(s)-h_{U}(s')|\le r\, d_{\mathcal D}(s,s')$. Taking $s\in\mathcal T$ with $d_{\mathcal D}(q(x),s)\le\varepsilon$ and applying the inequality to both the full support and the restricted support yields $\Phi_{[N]}(x)-\Phi_S(x)\le r\big(\Delta_S(s)+2\varepsilon\big)$. Minimizing over $x$ and maximizing over $s\in\mathcal T$ concludes the proof.
\end{proof}

\begin{remark}[Constructing $\mathcal T$ via Lipschitzness of $q$]\label{rmk:enet-const}
Let $\|\cdot\|$ be a norm on $\R^p$ with dual $\|\cdot\|_\ast$ and suppose $q$ is $L_q$-Lipschitz
on $(X,\|\cdot\|)$, such that $\|q(x)-q(x')\|\le L_q\|x-x'\|$.
Since $d_{\mathcal D}(s,s')\le \|D\|_\ast\,\|s-s'\|$ with $\|D\|_\ast:=\max_{i\in[N]_0}\|d_i\|_\ast$,
a grid in $X$ with step
\[
\Delta x\ \le\ \frac{\varepsilon}{\|D\|_\ast\,L_q}
\]
ensures that the image $\mathcal T=\{q(x): x\ \text{on the grid}\}$ is an $\varepsilon$-net in $d_{\mathcal D}$.
In the linear case $q(x)=M^\top x$, we have $L_q=\|M^\top\|$ (operator norm induced by $\|\cdot\|$).
An \emph{adaptive} version starts from the revealed directions $s_t=q(x_t)$ and refines partitions
until the $d_{\mathcal D}$-diameter is at most $\varepsilon$.
\end{remark}

In operational contexts, we maintain a set of evaluation directions $\mathcal T$ during the greedy procedure to compute the certificate $\widehat{\mathrm{gap}}(S;\mathcal T,\varepsilon) := r(\max_{s\in\mathcal T}\Delta_S(s) + 2\varepsilon)$. By Theorem~\ref{thm:enet-bridge}, this quantity deterministically upper-bounds the global value gap. While Theorem~\ref{thm:greedy} ensures that the greedy strategy improves the average coverage $F(S;\mathcal T)$, this certificate controls the worst-case deficit directly. In practice, we evaluate this certificate on sets built from revealed directions $\{q(x_t)\}$ and augmented with hold-out collections, reporting both the worst-case proxy and the average coverage.

\subsection{Finite Sample Feasibility}
\label{sec:oos-guarantees}

We complement the deterministic certificates with finite-sample out-of-sample guarantees assuming independent and identically distributed samples $\{u_j\}_{j=1}^m$ from the operational environment. Throughout this analysis, we use the atomic set definition $U_S(r) := \{u=\sum_{i\in S}\alpha_i d_i: \alpha_i\ge 0, \sum_{i\in S}\alpha_i\le r\}$.

\paragraph{Directional statistic.}
For each subset $S\subseteq[N]$, define the directional score as
\[
Z_S(u) := \max_{i\in S_0}\,\langle d_i,\,u\rangle .
\]
Calibrating the radius $r$ as a high quantile of $Z_S$ allows us to control the coverage level used by the pessimization oracle. Let $\widehat Q_{\beta}(Z_S)$ denote the empirical $\beta$-quantile of the scores $\{Z_S(u_j)\}_{j=1}^m$. For a confidence level $1-\delta$, define
\[
\eta := \sqrt{\frac{1}{2m}\,\log\!\Big(\frac{2\,\sum_{k=0}^{B}\binom{N}{k}}{\delta}\Big)}.
\]
Assume $\eta\le \alpha$. We then define the calibrated radius
\begin{equation}
\label{eq:rhat-dkw-union}
\widehat r(S)\ :=\ \widehat Q_{\,1-\alpha+\eta}(Z_S).
\end{equation}


\begin{theorem}[Finite-sample directional calibration]
\label{thm:oos-feas}
Assume $\eta \le \alpha$. With probability at least $1-\delta$ over the sampling, the calibration~\eqref{eq:rhat-dkw-union} ensures that, simultaneously for every subset $S \subseteq [N]$ with $|S| \le B$,
\[
\mathbb{P}\bigl(Z_S(U) > \hat r(S)\bigr) \le \alpha .
\]
\end{theorem}

The full proof is provided in Appendix~\ref{app:oos-proof}.

\subsection{Radius Calibration Procedure}
\label{subsec:integration}

The integration of this statistical layer into the greedy algorithm follows a sequential process. First, upon completion of the greedy selection which yields $S$, we compute the scores $Z(u_j)=\max_{i\in S_0}\langle d_i,u_j\rangle$ for all samples $j=1,\dots,m$. Second, we calibrate the radius $\widehat r$ according to \eqref{eq:rhat-dkw-union}. If the subset size is fixed to exactly $B$, the combinatorial term $\sum_{k=0}^B\binom{N}{k}$ may be replaced by $\binom{N}{B}$. Finally, we solve the robust optimization problem $\min_{x\in X}\Phi_S(x)$ using the uncertainty set $U_S(\widehat r)$ to obtain the solution $x_S^\star$. This procedure allows us to report both the finite-sample guarantee in Theorem~\ref{thm:oos-feas} and the value gap performance certificate.

\paragraph{Conformal alternative.}
The statistical correction term scales logarithmically as $O(\sqrt{\log \sum \binom{N}{k}/m})$, which becomes modest when the sample size $m$ is in the low thousands. For moderate sample sizes, a split-conformal calibration approach offers an alternative. By splitting the data into training (for selection) and an independent calibration set, one can define $\widehat r$ as the split-conformal quantile of $Z$, yielding exact marginal coverage for the event $Z_S(U)\le \widehat r$ without the logarithmic term, although this sacrifices uniformity over all data-dependent subsets.

\subsection{Validation Protocol}
\label{sec:validacion}

The validation protocol requires partitioning the available data into training and hold-out sets. The training set is used to determine the subset $S$ and, if applicable, to pre-calibrate the radius. The hold-out set is reserved to estimate the risk metric $\mathrm{Risk}(x_S^\star)$ for reporting purposes. Beyond the theoretical guarantees provided in Section~\ref{sec:oos-guarantees}, this protocol generates empirical summaries such as means, conditional value at risk, and violation rates. In our experiments, whenever calibration data is available, we apply the post-selection radius calibration step described above before reporting the final performance metrics.

\section{Greedy Selection Algorithm}

\subsection{Support Coverage Objective}

We begin by defining a finite set of evaluation directions $\mathcal{T}\subset\R^p$, which may consist of revealed directions $\{s_t\}$ or a validation set. The support coverage functional is defined as the average maximum alignment over this set:
\begin{equation}\label{eq:F}
F(S;\mathcal{T})\ :=\ \frac{1}{|\mathcal{T}|}\sum_{s\in \mathcal{T}}\max_{i\in S_0}\ip{d_i}{s}.
\end{equation}

\begin{proposition}[Submodularity and monotonicity of $F$]\label{prop:submod}
For any fixed direction $s\in\mathcal T$, the function $g_s(S):=\max_{i\in S_0}\ip{d_i}{s}$ is monotone and submodular with respect to the set $S$. Since $F$ is a non-negative linear combination of such functions, $F(\cdot;\mathcal{T})$ retains both monotonicity and submodularity.
\end{proposition}

This functional induces the surrogate design problem, denoted as $(\mathsf{P}_{\mathrm{cov}})$, which captures the directional coverage properties of a subset $S$ subject to a budget constraint:
\begin{equation}\label{eq:cov-design}
\tag{$\mathsf{P}_{\mathrm{cov}}$}
\max_{S\subseteq[N]}\ \Big\{ F(S;\mathcal T)\ :\ |S|\le B \Big\}.
\end{equation}
This problem abstracts the combinatorial core of our framework, where a larger value of $F(S;\mathcal T)$ corresponds to superior average support coverage on the directions relevant to the robust optimization task.

\subsection{Oracle Assisted Greedy Method}
\label{sec:alg}

To address the design problem efficiently, we propose an iterative procedure that alternates between optimizing the robust objective and expanding the set of evaluation directions. At each iteration, the algorithm identifies a worst case perturbation for the current solution using an oracle on the full model. This revealed direction is added to the evaluation set, guiding the subsequent greedy selection of the dictionary atom that maximizes the marginal gain in support coverage.

\begin{algorithm}[h]
\caption{Oracle assisted greedy for designing $U_S(r)$}
\label{alg:greedy}
\begin{algorithmic}[1]
   \STATE {\bfseries Input:} Dictionary $\mathcal{D}=\{d_i\}_{i=1}^N$, budget $B$, radius $r$ \emph{(later OOS calibrated to $\widehat r$ in Section~\ref{sec:oos-guarantees})}, optional validation set $\mathcal{V}$.
   \STATE $S\gets \emptyset$, \ $\mathcal{T}\gets \emptyset$.
   \FOR{$t=1,\dots,B$}
      \STATE Solve $x_t \in \Argmin_{x\in X}\ \Phi_S(x)$ \hfill (convex if applicable).
      \STATE \textbf{Oracle on the full dictionary:} obtain $u_t\in U_{[N]}(r)$ that approximately maximizes $f(x_t,u)$ \hfill (optional witness)
      \STATE With $f(x,u)=b(x)+\ip{u}{q(x)}$, set $s_t\leftarrow q(x_t)$ (used for greedy scoring), $A_t=\argmax_{i\in [N]_0}\ip{d_i}{s_t}$, and update $\mathcal{T}\gets \mathcal{T}\cup\{s_t\}$.
      \STATE Choose $i_t\in[N]\setminus S$ that maximizes the marginal gain of $F(S\cup\{i\};\mathcal{T})$.
      \STATE $S\gets S\cup\{i_t\}$.
      \STATE (Certifiable stopping) If $\Delta_S(s_t)=0$ and the best marginal gain in $F$ is zero, \textbf{stop}.
      \STATE (Optional) Early validation: stop if $\mathrm{Risk}(x_t)$ on $\mathcal{V}$ does not improve by more than $\varepsilon$.
   \ENDFOR
   \STATE \textbf{OOS calibration:} compute $\widehat r$ as in Section~\ref{sec:oos-guarantees} and re-solve $\min_{x\in X}\Phi_S(x)$ with $U_S(\widehat r)$ to obtain $x_S^\star$.
   \STATE {\bfseries Output:} $S$, $U_S(\widehat r)$, and $x_S^\star$.
\end{algorithmic}
\end{algorithm}

\begin{remark}[Worst case greedy baseline]\label{rmk:greedy-maxgap}
As a baseline that targets worst case coverage on the current direction set $\mathcal T$, one can replace the selection rule in Algorithm~\ref{alg:greedy} by
\[
i_t\ \in\ \argmin_{i\in[N]\setminus S}\ \max_{s\in\mathcal T}\Delta_{S\cup\{i\}}(s).
\]
This rule directly optimizes the certificate term $\max_{s\in\mathcal T}\Delta_S(s)$, and we use it as a comparison point in experiments.
\end{remark}

\subsection{Optimality Guarantees}

This section addresses the computational complexity of the design problem and the guarantees of the greedy approach. It first establishes computational hardness.

\begin{theorem}[NP hardness of directional coverage design]\label{thm:np-hard}
Given a finite dictionary $\mathcal D=\{d_i\}_{i=1}^N\subset\R^p$, a finite set $\mathcal T\subset\R^p$, and a budget $B$, problem~\eqref{eq:cov-design} is NP hard. This holds even when the coordinates of $\mathcal T$ and $\mathcal D$ are restricted to $\{0,1\}$.
\end{theorem}

\begin{proof}[\textit{Sketch (Full proof in App.~\ref{app:theom4_3})}]
The proof reduces from the classical \emph{Max Coverage} problem. Let $U=\{1,\dots,m\}$ be the ground set and $\{A_i\}_{i=1}^N$ be subsets of $U$. An instance of \eqref{eq:cov-design} is constructed by setting $p=m$ and identifying $\R^p$ with $\R^m$. Each $d_i$ is the indicator vector of $A_i$, and $\mathcal T$ is the set of canonical basis vectors. In this construction, maximizing $F(S;\mathcal T)$ is equivalent to maximizing the cardinality of the union of the selected subsets, which is exactly Max Coverage.
\end{proof}

The reduction is gap preserving, which yields a hardness of approximation result.

\begin{theorem}[Hardness of approximation]\label{thm:hardness-approx}
For every $\varepsilon>0$, unless $P=NP$, there is no polynomial time algorithm that, given an instance of \eqref{eq:cov-design}, produces a subset $S$ with $|S|\le B$ satisfying
\[
F(S;\mathcal T)\ \ge\ (1-1/e+\varepsilon)\,F(S^\star;\mathcal T),
\]
where $S^\star$ is an optimal solution.
\end{theorem}

\begin{proof}[\textit{Sketch (Full proof in App.~\ref{app:theom4_4})}]
Max Coverage is NP hard to approximate within any factor strictly larger than $1-1/e$. Since the reduction preserves objective values up to a constant factor, any algorithm that exceeds this ratio for \eqref{eq:cov-design} would contradict the known hardness of Max Coverage.
\end{proof}

Despite these hardness results, the submodularity of $F$ implies that the greedy algorithm attains the optimal approximation factor.

\begin{theorem}[$(1-1/e)$ guarantee for greedy on $F$]\label{thm:greedy}
Under the constraint $|S|\le B$, the greedy algorithm that iteratively adds the index with the largest marginal increase in $F(\cdot;\mathcal{T})$ satisfies
\[
F(S_{\mathrm{greedy}};\mathcal{T})\ \ge\ (1-1/e)\,F(S^\star;\mathcal{T}),
\]
where $S^\star$ maximizes $F(\cdot;\mathcal{T})$ subject to the budget constraint.
\end{theorem}

\begin{corollary}[Approximation optimality of the greedy scheme]\label{cor:greedy-optimal}
Combining Theorem~\ref{thm:hardness-approx} with Theorem~\ref{thm:greedy}, the greedy algorithm is approximation optimal for the directional coverage design problem under the standard assumption $P\neq NP$. No polynomial time algorithm achieves a strictly better worst case approximation factor.
\end{corollary}

\begin{remark}[Selection objective and certification]\label{rmk:selection-vs-cert}
The selection step optimizes the average coverage surrogate $F$ because it is monotone and submodular, which gives the tight greedy guarantee in Theorem~\ref{thm:greedy}. The robust value gap, however, is certified through the worst case deficit $\max_{s\in\mathcal T}\Delta_S(s)$ via the $\varepsilon$ net bridge in Theorem~\ref{thm:enet-bridge}. This distinction is explicit in the reported metrics: the algorithm optimizes $F$ and reports the certificate term a posteriori.

Appendix~\ref{app:average_worst} provides a finite probe set bridge between these quantities. In particular, Proposition~\ref{prop:avg_to_worst_bridge} bounds the worst case deficit on $\mathcal T$ by the surrogate gap $F([N];\mathcal T)-F(S;\mathcal T)$ plus a probe set occupancy term and a resolution term, and Corollary~\ref{cor:surrogate_to_value_gap} combines this bound with Theorem~\ref{thm:enet-bridge} to obtain a robust value gap bound in terms of the surrogate gap.
\end{remark}

\subsection{Computational Complexity}

The proposed method is tractable for standard convex problems. By the support representation in \eqref{eq:support}, the objective
\[
\Phi_S(x)=b(x)+r\max_{i\in S_0}\ip{d_i}{q(x)}
\]
allows the minimization in line 3 of Algorithm~\ref{alg:greedy} to be formulated as a convex problem whose size grows linearly with $|S|$ (for example, by linearizing the maximum in LP or SOCP instances). The pessimization step over $[N]$ reduces to support evaluation or an equivalent dual computation, and the selection step computes marginal gains of $F$ over $\mathcal T$, which is inexpensive once the probe set is fixed.

%% file: content/2_Experiments.tex
\input{content/2_1_Exp_Teacher}
\input{content/2_2_Exp_High}
\input{content/2_3_Exp_Uncertanty}

%% file: content/2_1_Exp_Teacher.tex
\section{Computational Experiments}
\label{sec:computational_exp}

Uncertainty directions come from a finite dictionary $\mathcal{D}=\{d_i\}_{i=1}^N \subset \mathbb{R}^p$. Given a budget $B$, each method selects a subset $S \subseteq \{1,\dots,N\}$ with $|S|\le B$, which defines the atomic uncertainty set $\mathcal{U}_S(r)$ (Section~\ref{sec:formulation}).

A common evaluation pipeline is used across settings. 
In Frozen Features, each method selects $S$ and evaluates certified robustness at a fixed threat radius $\varepsilon$ (Appendix~\ref{app:exp-details}) without radius calibration. Unless noted, results are averaged over random seeds, and calibrated settings use fixed $(\alpha,\delta)$.

For each method, the evaluation direction set $\mathcal{T}_{\mathrm{eval}}$ combines directions revealed by the oracle based procedure (Algorithm~\ref{alg:greedy}) with a held out pool when available. The reported metrics are the average coverage $F(S;\mathcal{T}_{\mathrm{eval}})$ and the proxy worst case
\[
G(S;\mathcal{T}_{\mathrm{eval}})\ :=\ r_{\mathrm{eval}} \max_{s\in\mathcal{T}_{\mathrm{eval}}}\Delta_S(s),
\]
where
\[
\Delta_S(s)=\max_{i\in[N]_0}\langle d_i,s\rangle-\max_{i\in S_0}\langle d_i,s\rangle,
\]
\[
S_0=S\cup\{0\},\ [N]_0=\{0,1,\dots,N\},\ d_0=0.
\]
The radius is set to $r_{\mathrm{eval}}=\widehat r$ in calibrated settings and $r_{\mathrm{eval}}=\varepsilon$ in Frozen Features.

Baselines include Greedy, Random, MaxNorm, and Greedy MaxGap, the minimax rule of Remark~\ref{rmk:greedy-maxgap} that targets $\max_{s\in\mathcal{T}}\Delta_S(s)$ on the current direction set. In Section~\ref{sec:synthetic}, we also report TopAct and a Mahalanobis ellipsoid fit. Random selects $B$ atoms uniformly without replacement from $\mathcal{D}$, and we report the mean and standard deviation over $K$ independent draws. Selection cost is reported through $B$, wall clock time, oracle calls, and $\rho=B/N$. Oracle methods use at most $B$ oracle queries by construction. Runtime includes robust solves and selection overhead under fixed solver and hardware settings (Appendix~\ref{app:exp-details}).

\subsection{Synthetic geometric validation}
\label{sec:synthetic}

This experiment isolates directional design in two dimensions. An X-shaped distribution is constructed and the dictionary is augmented with vertical decoy atoms that have large norm but low alignment with the evaluation direction. Uncertainty vectors $u\in\mathbb{R}^2$ are sampled near the diagonals, $\mathcal{D}$ is built from signal atoms and decoys, and the full dictionary support is compared with the restricted support induced by subsets $S$ of size $B$ in direction $x_{\mathrm{eval}}=(1,0)$ at unit radius.
Greedy Coverage and Greedy MaxGap are evaluated on a fixed direction pool $\mathcal{T}$, together with Random nested subsets, MaxNorm, and TopAct. A Mahalanobis ellipsoid fit on calibration samples is included as a convex baseline.

Figure~\ref{fig:synthetic-gap} reports the support error $z_{\mathrm{full}}-z_S$ in direction $x_{\mathrm{eval}}$. Greedy Coverage and TopAct reduce the error at small budgets by selecting atoms with large projection onto $x_{\mathrm{eval}}$. Random improves more slowly because part of the budget is spent on decoys. MaxNorm performs poorly because decoys have large norms but contribute little support in direction $x_{\mathrm{eval}}$. Greedy MaxGap reduces a proxy worst case criterion on $\mathcal{T}$ and typically requires larger budgets to match average coverage.

\begin{figure}[t]
  \centering
  \includegraphics[width=0.9\linewidth]{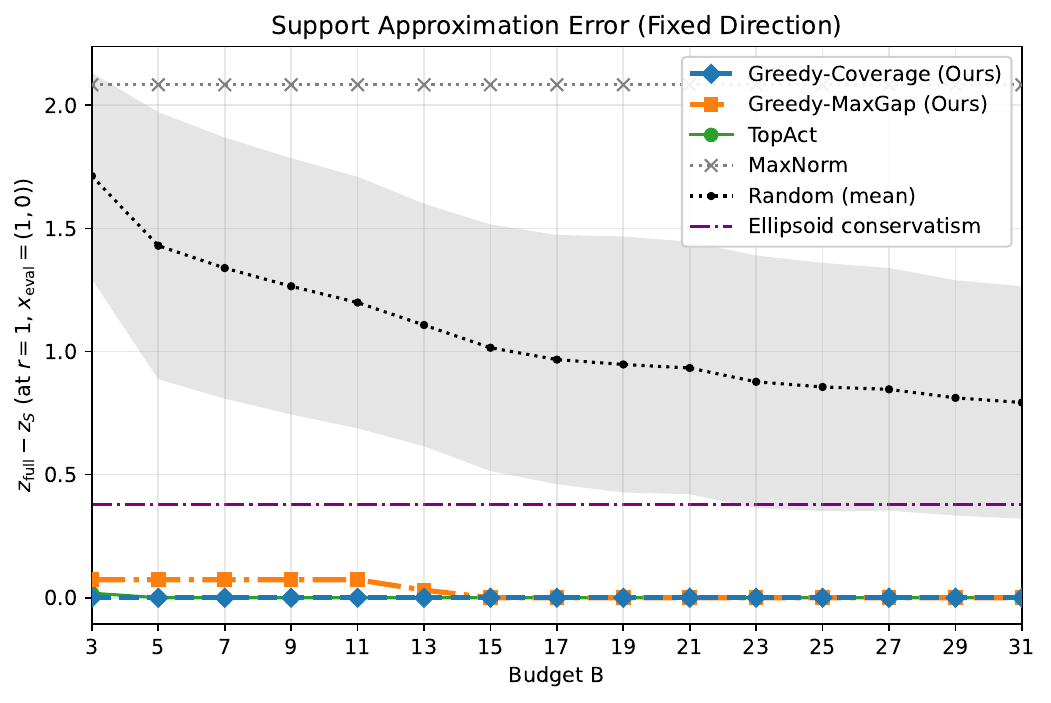}
  \caption{\textbf{Support approximation error.}
  Error versus budget $B$ in direction $x_{\mathrm{eval}}$.}
  \label{fig:synthetic-gap}
\end{figure}

\begin{figure}[t]
  \centering
  \includegraphics[width=0.8\linewidth]{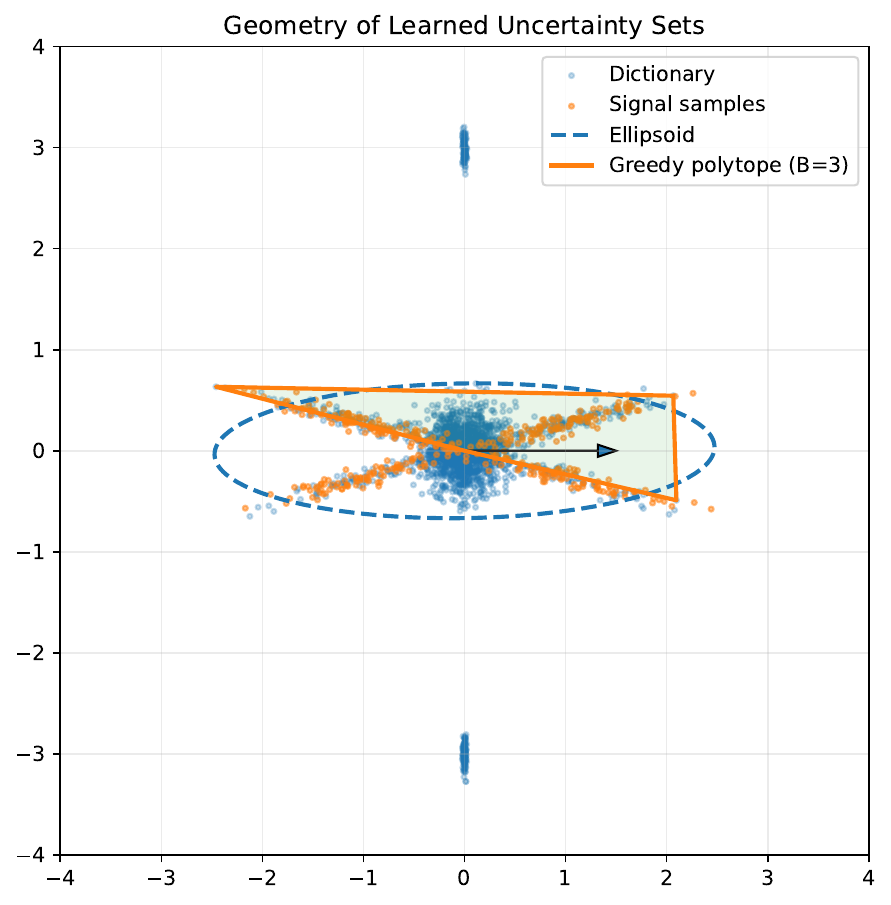}
  \caption{\textbf{Geometry of learned uncertainty sets.}
  Atomic polytope induced by $S$ versus an ellipsoidal baseline.}
  \label{fig:synthetic-geometry}
\end{figure}

Figure~\ref{fig:synthetic-geometry} illustrates the mechanism. With small $B$, the atomic polytope can be anisotropic because it uses a small set of selected atoms to match support in relevant directions. The ellipsoidal baseline summarizes the distribution globally and can be conservative in direction $x_{\mathrm{eval}}$ when decoys inflate dispersion.

\begin{figure}[!h]
  \centering
  \includegraphics[width=0.9\linewidth]{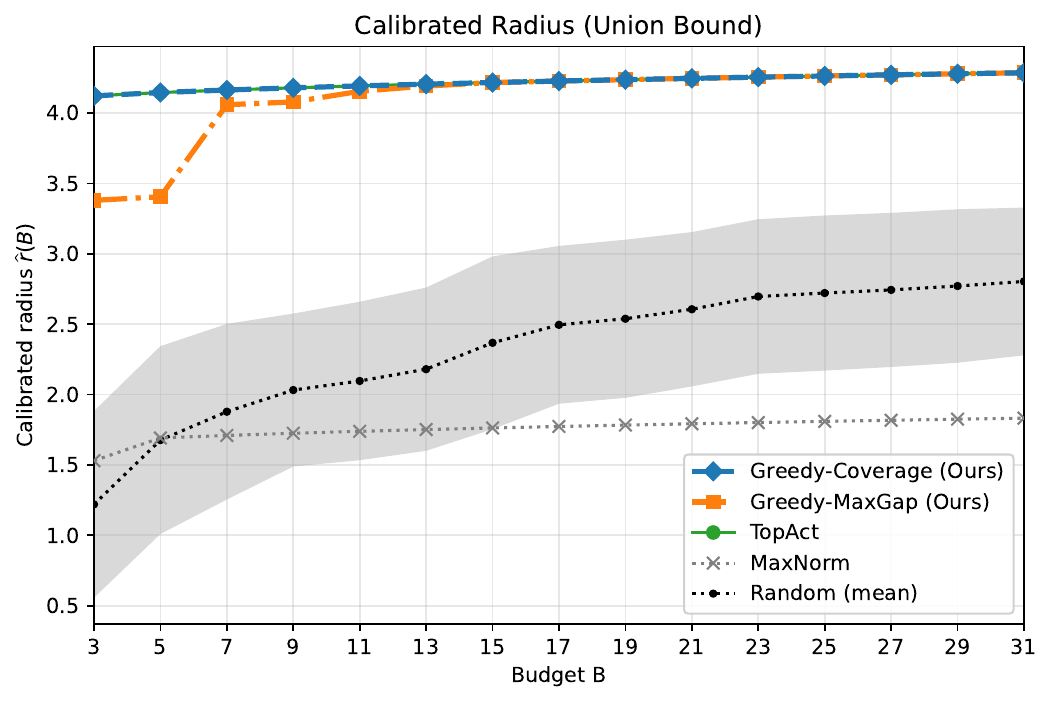}
  \caption{\textbf{Calibrated radius.}
  Calibrated $\widehat r(B)$ versus budget $B$ under DKW union calibration.}
  \label{fig:synthetic-radius}
\end{figure}

\clearpage

Figure~\ref{fig:synthetic-radius} reports $r_b(B)$. As $B$ increases, the union bound correction grows with the number of candidate subsets, so $r_b(B)$ increases after the empirical quantile stabilizes. Appendix~\ref{app:exp-details} compares DKW union and split calibration and reports sweeps over $(\alpha,\delta)$ with test split violations. DKW union remains conservative, with mean violation gap $\mathrm{viol}_d-\alpha \in [-0.0086,-0.0082]$ across the tested $\delta$, while split calibration is mildly anti conservative, with mean gap in $[0.0037,0.0048]$; see Tables~\ref{tab:f1_calibration_method_delta} and~\ref{tab:f1_calibration_budget_method} and Figure~\ref{fig:f1_calibration_conservatism}.

%% file: content/2_2_Exp_High.tex
\subsection{Scaling of certified coverage with large dictionaries}
\label{sec:scaling}

We study directional selection with large uncertainty dictionaries in high ambient dimension. To keep training convex, we train linear robust classifiers on frozen features extracted from a pretrained convolutional network. We consider binary tasks built from CIFAR-10 and CIFAR-100, and report results in two regimes: Hard tasks from CIFAR-10 and Expert tasks from CIFAR-100. The appendix lists all task pairs and additional diagnostics.

For each task and seed, we split data into disjoint sets for fitting, direction selection, reporting, and final testing. We mine a candidate dictionary $\mathcal D=\{d_i\}_{i=1}^N$ by projected gradient ascent in feature space under an $\ell_2$ radius $\varepsilon$. We evaluate scaling with dictionary size by measuring coverage as a function of the budget density $\rho=B/N$.

Figure~\ref{fig:scaling_cov_vs_rho} reports coverage versus $\rho$ for multiple values of $N$. At matched density, Greedy maintains coverage as $N$ increases, while Random degrades. This gap grows with $N$, consistent with a haystack effect in which random selection allocates budget to directions that do not contribute to the coverage objective.

\begin{figure}[h]
  \centering
  \includegraphics[width=\linewidth]{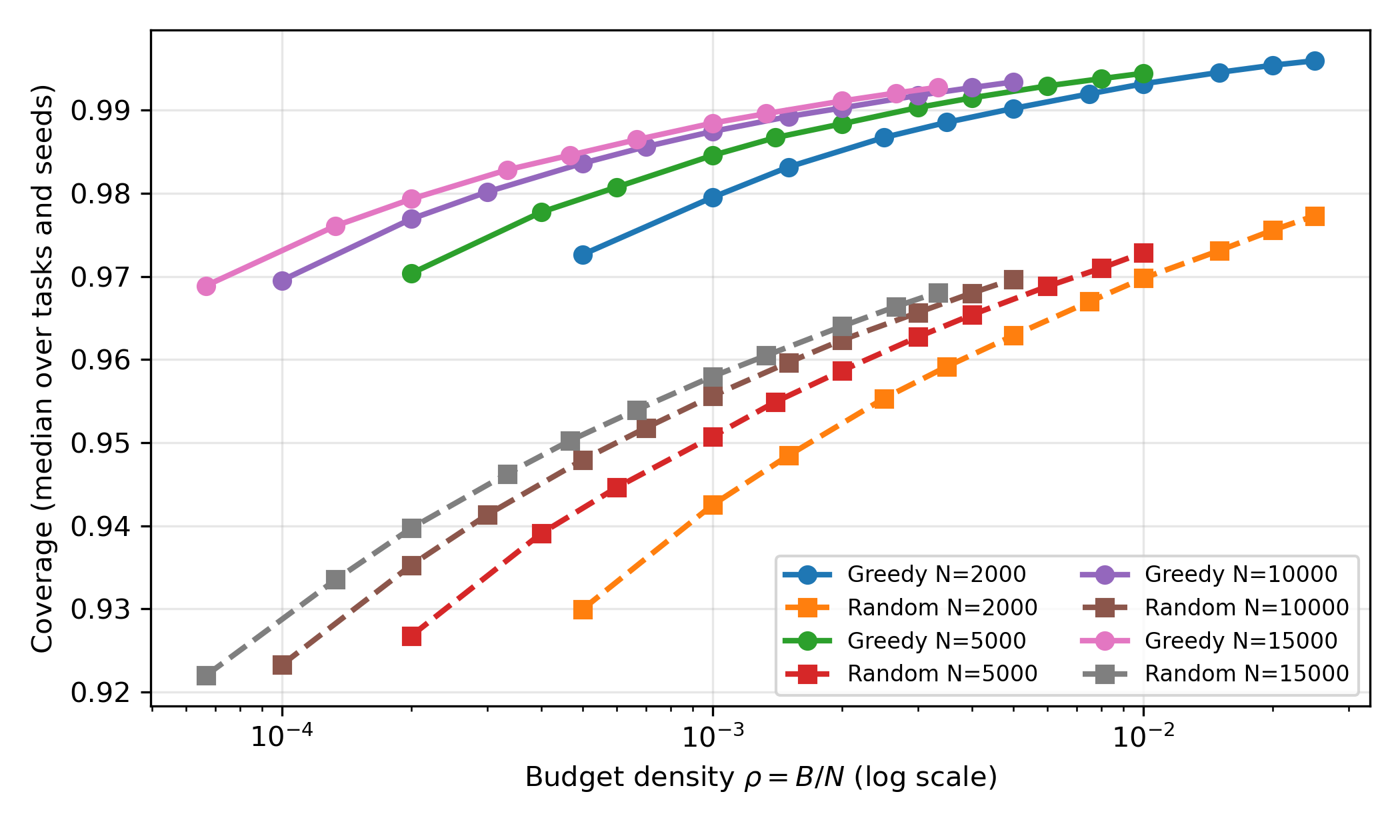}
  \caption{\textbf{Coverage versus budget density.}
  Median coverage over tasks and seeds as a function of $\rho=B/N$, shown for multiple dictionary sizes $N$.}
  \label{fig:scaling_cov_vs_rho}
\end{figure}

To express these trends as a cost, we fix a target coverage level and report the smallest budget $B$ that reaches the target. Figure~\ref{fig:scaling_minB} shows that Greedy reaches coverage $\ge 0.990$ with budgets that increase slowly with $N$, while Random does not reach the target within the tested range $B\le 50$. Table~\ref{tab:minB_target} lists the corresponding budgets and reports an extrapolated estimate for Random, since it does not reach the target for $B\le B_{\max}$.

\begin{figure}[h]
  \centering
  \includegraphics[width=\linewidth]{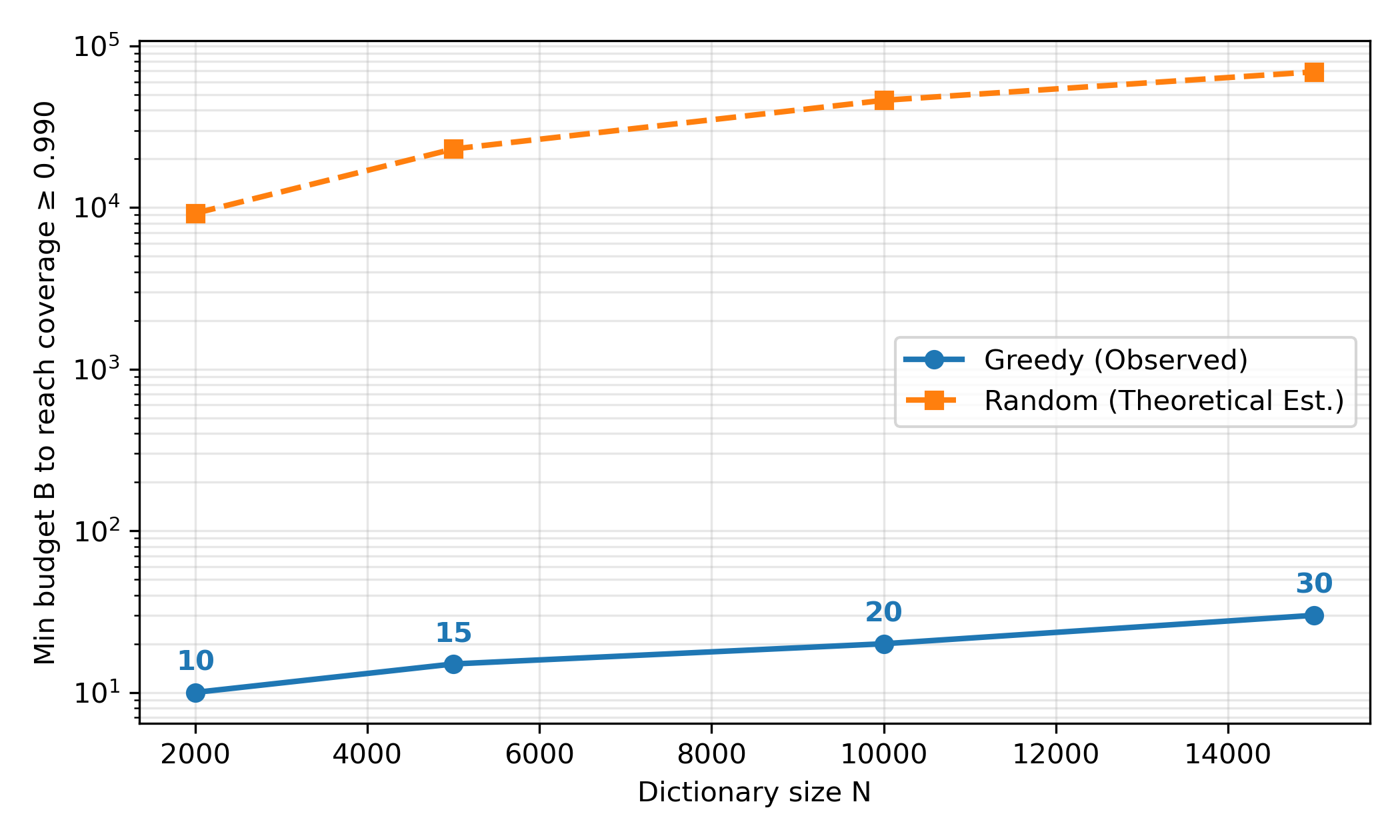}
  \caption{
  Minimum $B$ needed to reach coverage $\ge 0.990$ as a function of dictionary size $N$. Random does not reach the target for $B\le 50$.}
  \label{fig:scaling_minB}
\end{figure}

\begin{table}[ht] 
\centering
\caption{\textbf{Minimum budget for target coverage ($\tau=0.99$).}
Greedy reaches the target with small budgets ($B \le 20$), while Random requires orders of magnitude more (extrapolated estimates $^\ast$).}
\label{tab:minB_target}
\vspace{4pt} 
\resizebox{\linewidth}{!}{ 
\begin{tabular}{@{}l c c r r@{}} 
\toprule
\textbf{Dict. Size} ($N$) & \textbf{Target} & \textbf{$B_{\max}$} & \textbf{Greedy $B$} & \textbf{Random $B$ (Est.)} \\
\midrule
2,000  & 0.99 & 50 & \textbf{10} & $9,208^\ast$ \\
5,000  & 0.99 & 50 & \textbf{15} & $23,023^\ast$ \\
10,000 & 0.99 & 50 & \textbf{20} & $46,049^\ast$ \\
15,000 & 0.99 & 50 & \textbf{30} & $69,075^\ast$ \\

\bottomrule
\end{tabular}
}
\vspace{2pt}
\scriptsize{$^\ast$ Extrapolated via coupon collector approximation.}
\end{table}

%% file: content/2_3_Exp_Uncertanty.tex
\begin{table*}[t]
\centering\small
\caption{Deterministic robustness rate across budgets: percentage of the 64 seeds feasible under the full dictionary baseline with zero violations.}
\label{tab:det_success_oral}
\begin{tabular}{lrrrrrrrrr}
\toprule
Policy & $B{=}1$ & 2 & 3 & 5 & 7 & 10 & 15 & 20 & 30 \\
\midrule
greedy-maxgap   & 40.6 & 51.6 & 56.2 & 59.4 & 56.2 & 54.7 & 53.1 & 42.2 & 39.1 \\
greedy-coverage & 40.6 & 54.7 & 59.4 & 62.5 & 62.5 & 54.7 & 45.3 & 34.4 & 34.4 \\
max-gamma       & 28.1 & 26.6 & 29.7 & 37.5 & 46.9 & 42.2 & 35.9 & 29.7 & 31.2 \\
\bottomrule
\end{tabular}
\end{table*}

\subsection{Budget Efficiency and Reliability Analysis}
\label{sec:synthetic_results}

The synthetic construction follows \citet{LozanoBorrero2025} with modifications for uncertainty set selection. The dimensions are $n=q=30$ and $m=p=3$, and problem tightness is controlled by $\texttt{rhs}=0.9$, where lower values correspond to tighter constraints. Each instance is generated from a random seed.

A total of 100 instances are generated, of which 64 instances (64.0\%) are robust feasible under the full dictionary baseline, with zero violations. The onset budget and budget-wise efficiency analyses focus on this subset to isolate budget efficiency from baseline infeasibility. Appendix~\ref{app:all100_sec53} reports certification counts over all 100 instances.

Three deterministic policies (greedy-maxgap, greedy-coverage, and max-gamma) are compared with a random baseline. Efficiency is measured by the onset budget $B^\star$, the minimum budget that achieves zero violations.

Table~\ref{tab:det_success_oral} reports zero-violation rates across budgets, and Table~\ref{tab:Bstar_oral} summarizes $B^\star$ over the same 64 instances. Success rates are not monotonic in $B$ because each budget level is solved independently and the selected sets are not nested. Both greedy-maxgap and greedy-coverage reach a median onset of $B^\star=2$, compared to 7 for max-gamma; within the evaluated budget grid, greedy-maxgap certifies 58 of 64 instances.

\begin{table}[!h]
\centering\small
\setlength{\tabcolsep}{4pt}
\caption{Onset of certified robustness ($B^\star$) on seeds feasible under the full dictionary baseline ($N=64$). Robust Rate reports the number of instances with zero violations within the evaluated budget grid. Mean and standard deviation are computed over instances that achieve robustness.}
\label{tab:Bstar_oral}
\begin{tabular}{lccc}
\toprule
Policy & Robust Rate & Median $B^\star$ & Mean $\pm$ Std. \\
\midrule
greedy-maxgap   & 58/64 & 2 & $6.3 \pm 8.3$ \\
greedy-coverage & 55/64 & 2 & $6.3 \pm 9.5$ \\
max-gamma       & 53/64 & 7 & $9.7 \pm 10.0$ \\
\bottomrule
\end{tabular}
\end{table}

The random baseline has no deterministic guarantee at a fixed budget, so reliability is evaluated with 50 independent repetitions per instance. Table~\ref{tab:random_patterns_oral} reports the mean success rate across the 64 instances feasible under the full dictionary baseline, together with the standard deviation across instances. The dispersion remains large across budgets, including larger values of $B$.

\begin{table}[!h]
\centering\small
\caption{Reliability of the Random baseline on instances feasible under the full dictionary baseline ($N=64$). For each budget $B$, success is the fraction of repetitions with zero violations, and the table reports the mean and standard deviation across instances.}
\label{tab:random_patterns_oral}
\begin{tabular}{lrrrrr}
\toprule
B & 3 & 5 & 7 & 10 & 15 \\
\midrule
Reps/seed & 50 & 50 & 50 & 50 & 50 \\
Success (\%) & 22.9 & 37.2 & 46.7 & 60.1 & 69.8 \\
Std & $\pm$ 22.8 & $\pm$ 29.8 & $\pm$ 33.1 & $\pm$ 33.6 & $\pm$ 33.1 \\
\bottomrule
\end{tabular}
\end{table}

\section{Discussion}
\label{sec:discussion}

\noindent\textbf{How does informed selection compare to uninformed baselines?}
Across geometric validation (Section~\ref{sec:synthetic}) and scaling benchmarks (Section~\ref{sec:scaling}), greedy policies reduce support approximation error by aligning selected atoms with the evaluation directions that drive the support function. Random selection spends budget on decoys or irrelevant constraints (Figure~\ref{fig:synthetic-gap}), weakening coverage and amplifying the haystack effect as the dictionary grows (Figure~\ref{fig:scaling_cov_vs_rho}).

\noindent\textbf{How should the $\varepsilon$ net term in the value certificate be interpreted?}
Theorem~\ref{thm:enet-bridge} decomposes the robust value gap into a discrete alignment deficit on a finite probe set $T$ and a discretization term $2\varepsilon$. Practical performance depends on the alignment term $\max_{s\in T}\Delta_S(s)$ and on how well $T$ approximates $q(X)$. Appendix~\ref{app:epsilon_practicality} introduces a proxy $\hat{\varepsilon}(T)$ and a stopping rule based on the empirical certificate $\max_{s\in T}\Delta_S(s) + 2\hat{\varepsilon}(T)$, making the bound usable with finite probe pools. Appendix~\ref{sec:greedy-inexact} shows that this certificate remains informative under inexact selection and inexact oracle evaluations.

\noindent\textbf{Does this efficiency translate to integer robust optimization?}
In integer programming (Section~\ref{sec:synthetic_results}), Greedy MaxGap certifies robustness for most robust feasible instances with smaller budgets than max-gamma (Table~\ref{tab:Bstar_oral}). This pattern indicates that feasibility is often driven by a small set of active constraints that the greedy objective identifies.

\noindent\textbf{Why are success rates not monotonic?}
Success rates fluctuate because each budget is solved independently (Table~\ref{tab:det_success_oral}). Larger budgets explore new combinations of scenarios rather than extending earlier selections, so rates can decrease at some levels and improve at others. A nested budget schedule would recover monotonicity if required.

\noindent\textbf{Is random selection a viable heuristic?}
Table~\ref{tab:random_patterns_oral} shows substantial variance for random selection even at larger budgets, with inconsistent certification across runs. Random selection is a high variance heuristic, while greedy policies provide deterministic guarantees with small sets.


%% file: content/3_Conclusion.tex
\section{Conclusions}
\label{sec:conclusion}

The paper presents a framework for directional uncertainty sets from a finite dictionary. It formulates design as a submodular coverage problem, yielding a greedy method with approximation guarantees and tractability. The framework also provides finite-sample guarantees for radius calibration. Experiments show robust direction selection from large dictionaries, stable coverage in high dimensions, and small budget certification in integer programs.

%% file: content/0_appendix.tex
\newpage
\appendix
\onecolumn
\title{Which Directions Matter? Sparse Design for Affine Robust Optimization\\(Supplementary Material)}
\maketitle

\input{content/3_Related}

\section{Scope and limitations.}
\label{sec:limitations}

Our framework focuses on affine robust optimization with finite dictionaries. While the design problem is NP-hard, our submodular formulation ensures the greedy strategy provides a strong approximation. The primary limitation is the dependence on the input dictionary; if the candidate pool lacks the true worst-case directions, the resulting set will be optimistic. Therefore, the method is best combined with high-recall mining procedures or domain knowledge. Future work will extend this approach to end-to-end dictionary learning and non-affine uncertainty structures.

\section{Broader impact.}
This framework supports reliability oriented decision making in settings where structured perturbations matter, such as planning and risk sensitive prediction. It can increase conservatism and reduce performance for some groups if the dictionary reflects biased data or incomplete domain knowledge. The dictionary, budget, and calibration target encode which shifts are treated as relevant. We recommend documenting these choices, reporting selected directions, sensitivity to budget and radius, and evaluating coverage on held out data.

\input{content/0_appendix_theo_1}

\input{content/0_appendix_theo_2}
\input{content/0_appendix_exp}

%% file: content/3_Related.tex
\section{Related Work}

\subsection{Scenario-Based Robust Optimization}
Robust optimization studies worst case decision making under uncertainty and provides tractable reformulations for broad classes of programs through uncertainty sets and their support functions \citep{BenTal1998, ElGhaoui1997}.
When uncertainty is complex, a common alternative is to approximate it with a finite set of representative scenarios. This scenario approach samples constraints and yields finite sample guarantees for chance constrained problems \citep{Calafiore2006, Campi2008}.
Recent work studies scenario reduction for one and two stage robust optimization with discrete uncertainty and derives approximation guarantees that hold for any reduced uncertainty set \citep{goerigk2023optimal}. Related work proposes data-driven rules to identify relevant scenarios that produce strong bounds early in iterative constraint or variable generation methods for robust combinatorial optimization \citep{goerigk2025data}.
These results provide certificates for probabilistic constraint satisfaction or for approximation quality and convergence speed under scenario reduction. They do not provide selection aware certificates that quantify how a reduced representation deviates from a richer directional robust model, such as computable bounds on the robust value gap induced by selecting a subset of directions.

\subsection{Adversarial Robust Learning}
Robust optimization has influenced robust learning in several forms. Robust support vector machines relate robust losses to regularization \citep{Xu2009}, and distributionally robust learning formalizes robustness to distribution shift \citep{Namkoong2016, Sinha2018}.
The study of adversarial examples \citep{Szegedy2014} led to adversarial training, which can be written as a robust optimization problem over norm bounded perturbations \citep{Goodfellow2015, Madry2018}. To reduce cost or conservatism, structured perturbation families have been considered, including geometric transformations \citep{Engstrom2019}.
Certified defenses provide worst case guarantees for specific perturbation classes, including convex relaxation methods \citep{Wong2018}, semidefinite bounds \citep{Raghunathan2018}, and randomized smoothing \citep{Cohen2019}.
Most of this literature treats the perturbation family as fixed or learned end to end. In contrast, the present setting selects a small subset from a large, pre specified dictionary of candidate directions and provides an explicit certificate for the approximation loss relative to the full directional model.

\subsection{Submodular Subset Selection}
The selection problem is closely connected to submodular optimization. Many coverage objectives are monotone and submodular, and greedy algorithms achieve approximation guarantees in this setting. The classical result of \cite{Nemhauser1978} shows that greedy attains a $(1-1/e)$ factor for monotone submodular maximization under a cardinality constraint, and that this factor is tight under standard complexity assumptions.
Submodularity has been used for data subset selection and active learning \citep{Wei2015}, as well as for large scale variants based on distributed and streaming computation \citep{Mirzasoleiman2015}. Surveys also cover sensing and information gathering objectives \citep{Krause2014}. Related work on coresets selects representative points for learning objectives \citep{Sener2018}.
The central difficulty here is not only to optimize a combinatorial objective, but to identify a coverage functional whose value directly upper bounds robust degradation. The selected objects are uncertainty directions rather than data points, while the resulting optimization follows the same greedy principle once the appropriate functional is defined.

\subsection{Calibration and Conformal Guarantees}
After selecting directions, the radius of the uncertainty set must be calibrated. Small radii may under protect, whereas large radii may be overly conservative. Calibration of probabilistic outputs is widely studied in prediction systems \citep{Zadrozny2002, Guo2017}.
Conformal prediction provides distribution free finite sample coverage guarantees by calibrating a score on held out data \citep{Vovk2005, Shafer2008}, with developments for regression and quantile based intervals \citep{Lei2018, Romano2019}.
Conformal robust optimization yields finite sample coverage certificates for uncertainty sets, but typically does not address the representation problem of selecting which candidate directions to retain, nor provide a deterministic certificate for how selection changes the robust objective. The radius calibration in this work targets a prescribed out-of-sample coverage level for the robust solution, complementing the deterministic selection aware certificate.

%% file: content/0_appendix_theo_1.tex
\section{Proof of Theorems}


\subsection{Proof of Proposition~\ref{prop:gap-ptwise}}
\label{app:prop2_3}

\begin{proof}
Recall
\[
\Phi_S(x)=b(x)+r\max_{i\in S_0}\langle d_i,q(x)\rangle,
\qquad
\Phi_{[N]}(x)=b(x)+r\max_{i\in [N]_0}\langle d_i,q(x)\rangle.
\]
Therefore, for every $x\in X$,
\begin{align*}
\Phi_{[N]}(x)-\Phi_S(x)
&=r\left(\max_{i\in [N]_0}\langle d_i,q(x)\rangle-\max_{i\in S_0}\langle d_i,q(x)\rangle\right)\\
&=r\,\Delta_S(q(x)),
\end{align*}
which is the claimed identity.
\end{proof}

\subsection{Proof of Corollary~\ref{cor:igualdad-valores}}
\label{app:cor2_4}

\begin{proof}
Let $x_S^\star\in\arg\min_{x}\Phi_S(x)$. By Proposition~\ref{prop:gap-ptwise},
\[
\Phi_{[N]}(x_S^\star)-\Phi_S(x_S^\star)=r\,\Delta_S(q(x_S^\star)).
\]
If $\Delta_S(q(x_S^\star))=0$, then $\Phi_{[N]}(x_S^\star)=\Phi_S(x_S^\star)$.

Also, $\Delta_S(\cdot)\ge 0$, so Proposition~\ref{prop:gap-ptwise} implies
\[
\Phi_{[N]}(x)\ge \Phi_S(x)\qquad\forall x\in X.
\]
Hence
\[
\min_x \Phi_{[N]}(x)\ge \min_x \Phi_S(x)=\Phi_S(x_S^\star)=\Phi_{[N]}(x_S^\star)\ge \min_x \Phi_{[N]}(x),
\]
which proves
\[
\min_x \Phi_{[N]}(x)=\min_x \Phi_S(x).
\]
\end{proof}


\subsection{Proof of Theorem~\ref{thm:enet-bridge}}
\label{app:theom3_1}

\begin{proof}
For any index set $A\subseteq [N]_0$, define
\[
M_A(s) := \max_{i\in A}\ip{d_i}{s}, \qquad s\in\R^p.
\]
Then $\Delta_S(s)=M_{[N]_0}(s)-M_{S_0}(s)$.

We first record a Lipschitz property (with respect to $d_D$): for any $A\subseteq [N]_0$ and any $s,s'\in\R^p$,
\begin{align*}
|M_A(s)-M_A(s')|
&\le \max_{i\in A}\left|\ip{d_i}{s-s'}\right|
\le \max_{i\in [N]_0}\left|\ip{d_i}{s-s'}\right|
= d_D(s,s').
\end{align*}
Indeed, for any $i\in A$, $\ip{d_i}{s}\le \ip{d_i}{s'}+| \ip{d_i}{s-s'}|$, so taking maxima over $i\in A$ gives
$M_A(s)\le M_A(s')+d_D(s,s')$; exchanging $s,s'$ yields the reverse inequality.

Now fix an arbitrary $x\in X$. Since $T$ is an $\varepsilon$-net of $q(X)$ in $d_D$, there exists $t=t(x)\in T$ such that
\[
d_D(q(x),t)\le \varepsilon.
\]
Using the previous Lipschitz bound for both $A=[N]_0$ and $A=S_0$,
\begin{align*}
\Delta_S(q(x))
&= M_{[N]_0}(q(x)) - M_{S_0}(q(x))\\
&\le \bigl(M_{[N]_0}(t)+\varepsilon\bigr) - \bigl(M_{S_0}(t)-\varepsilon\bigr)\\
&= \Delta_S(t)+2\varepsilon\\
&\le \max_{s\in T}\Delta_S(s)+2\varepsilon.
\end{align*}

By Proposition~\ref{prop:gap-ptwise} (pointwise value gap), for every $x\in X$,
\[
\Phi_{[N]}(x)-\Phi_S(x) = r\,\Delta_S(q(x)).
\]
Therefore,
\[
\Phi_{[N]}(x)-\Phi_S(x)
\le
r\left(\max_{s\in T}\Delta_S(s)+2\varepsilon\right).
\]
This bound holds for every $x\in X$. In particular, if $x_S^\star\in\arg\min_x\Phi_S(x)$, then
\begin{align*}
\min_x \Phi_{[N]}(x)-\min_x \Phi_S(x)
&\le \Phi_{[N]}(x_S^\star)-\Phi_S(x_S^\star)\\
&\le r\left(\max_{s\in T}\Delta_S(s)+2\varepsilon\right),
\end{align*}
which proves the theorem.
\end{proof}

\subsection{Proof of Theorem~\ref{thm:oos-feas}}
\label{app:oos-proof}

\begin{proof}
Let
\[
M_B := \sum_{k=0}^{B}\binom{N}{k}.
\]
Fix a subset $S\subseteq[N]$ with $|S|\le B$. Let $F_S$ denote the population CDF of $Z_S(U)$, and let $\widehat F_S$ be its empirical CDF based on $\{U_j\}_{j=1}^m$.

By the Dvoretzky--Kiefer--Wolfowitz (DKW) inequality, for each fixed $S$,
\[
\mathbb{P}\!\left(\sup_{z\in\R}\bigl|\widehat F_S(z)-F_S(z)\bigr|>\eta\right)
\le 2e^{-2m\eta^2}
= \frac{\delta}{M_B}.
\]
Define the event
\[
\E_S := \left\{\sup_{z\in\R}\bigl|\widehat F_S(z)-F_S(z)\bigr|\le \eta\right\}.
\]
Then $\mathbb{P}(\E_S)\ge 1-\delta/M_B$.

Now take a union bound over all subsets $S\subseteq[N]$ with $|S|\le B$. The event
\[
\E := \bigcap_{\substack{S\subseteq[N]\\|S|\le B}}\E_S
\]
satisfies
\[
\mathbb{P}(\E)\ge 1-\sum_{\substack{S\subseteq[N]\\|S|\le B}}\frac{\delta}{M_B}\ge 1-\delta.
\]

We now prove the claimed guarantee on the event $\E$.
Fix any $S\subseteq[N]$ with $|S|\le B$, and set
\[
r:=\widehat r(S)=\widehat Q_{\,1-\alpha+\eta}(Z_S).
\]
By definition of the empirical quantile,
\[
\widehat F_S(r)\ge 1-\alpha+\eta.
\]
Since $\E$ implies $\widehat F_S(z)-F_S(z)\le \eta$ for all $z$, we obtain
\[
F_S(r)\ge \widehat F_S(r)-\eta \ge 1-\alpha.
\]
Therefore,
\[
\mathbb{P}\!\left(Z_S(U)>r\right)=1-F_S(r)\le \alpha.
\]
Because this argument is valid for every $S$ with $|S|\le B$, the guarantee holds simultaneously for all such subsets on the event $\E$.

Finally, the subset returned by Algorithm~\ref{alg:greedy} is data-dependent, but it belongs to the same finite family $\{S\subseteq[N]: |S|\le B\}$. Hence the same simultaneous guarantee applies to the algorithm output.
\end{proof}

\subsection{Proof of Proposition~\ref{prop:submod}}
\label{app:prop4_1}

\begin{proof}
Fix $s\in T$ and define
\[
g_s(S):=\max_{i\in S_0}\langle d_i,s\rangle.
\]
Monotonicity is immediate: if $A\subseteq B$, then $A_0\subseteq B_0$, hence
\[
g_s(A)=\max_{i\in A_0}\langle d_i,s\rangle
\le
\max_{i\in B_0}\langle d_i,s\rangle
=g_s(B).
\]

To show submodularity, let $A\subseteq B\subseteq [N]$ and $j\in [N]\setminus B$. Write $a_j:=\langle d_j,s\rangle$. Then
\[
g_s(A\cup\{j\})-g_s(A)=\max\{g_s(A),a_j\}-g_s(A)=\max\{0,a_j-g_s(A)\},
\]
and similarly
\[
g_s(B\cup\{j\})-g_s(B)=\max\{0,a_j-g_s(B)\}.
\]
Since $g_s(A)\le g_s(B)$, we get
\[
\max\{0,a_j-g_s(A)\}\ge \max\{0,a_j-g_s(B)\},
\]
which is the diminishing-returns property. Thus $g_s$ is submodular.

Finally,
\[
F(S;T)=\frac1{|T|}\sum_{s\in T} g_s(S)
\]
is a nonnegative linear combination of monotone submodular functions, hence is itself monotone and submodular.
\end{proof}


\subsection{Proof of Theorem~\ref{thm:np-hard}}
\label{app:theom4_3}

\begin{proof}
We reduce from the classical \emph{Max-Coverage} problem.

An instance of Max-Coverage consists of:
\begin{itemize}
    \item a ground set $U=\{1,\dots,m\}$,
    \item subsets $A_1,\dots,A_N\subseteq U$,
    \item a budget $B$,
\end{itemize}
and asks for a set of indices $S\subseteq[N]$ with $|S|\le B$ maximizing
\[
\left|\bigcup_{i\in S}A_i\right|.
\]

Given such an instance, construct an instance of the directional coverage problem as follows:
\begin{itemize}
    \item Set the ambient dimension to $p=m$.
    \item For each $i\in[N]$, define $d_i\in\{0,1\}^m$ as the indicator vector of $A_i$:
    \[
    (d_i)_j :=
    \begin{cases}
    1, & j\in A_i,\\
    0, & j\notin A_i.
    \end{cases}
    \]
    \item Let
    \[
    T:=\{e_1,\dots,e_m\}\subseteq\R^m,
    \]
    where $e_j$ is the $j$-th canonical basis vector.
\end{itemize}
This reduction is computable in polynomial time, and all coordinates of $d_i$ and $T$ are in $\{0,1\}$.

Now fix any $S\subseteq[N]$. For any $s=e_j\in T$,
\[
\max_{i\in S_0}\ip{d_i}{e_j}
=
\max_{i\in S_0}(d_i)_j.
\]
Since each $(d_i)_j\in\{0,1\}$ and $d_0=0$, we have
\[
\max_{i\in S_0}(d_i)_j =
\begin{cases}
1, & \text{if there exists }i\in S\text{ such that }j\in A_i,\\
0, & \text{otherwise}.
\end{cases}
\]
Therefore,
\[
\sum_{s\in T}\max_{i\in S_0}\ip{d_i}{s}
=
\sum_{j=1}^m \mathbf{1}\!\left\{j\in \bigcup_{i\in S}A_i\right\}
=
\left|\bigcup_{i\in S}A_i\right|.
\]
Dividing by $|T|=m$ gives
\[
F(S;T)=\frac1m\left|\bigcup_{i\in S}A_i\right|.
\]

Thus, maximizing $F(S;T)$ subject to $|S|\le B$ is exactly equivalent (up to the constant factor $1/m$) to solving the Max-Coverage instance. Since Max-Coverage is NP-hard, the directional coverage design problem is NP-hard as well.
\end{proof}


\subsection{Proof of Theorem~\ref{thm:hardness-approx}}
\label{app:theom4_4}

\begin{proof}
We use the same reduction from Max-Coverage as in the proof of Theorem~\ref{thm:np-hard}.

Let
\[
\operatorname{cov}(S):=\left|\bigcup_{i\in S}A_i\right|
\]
denote the Max-Coverage objective value. For the constructed instance of $(P_{\mathrm{cov}})$, we proved that for every $S\subseteq[N]$,
\[
F(S;T)=\frac{1}{m}\operatorname{cov}(S).
\]
Hence the reduction preserves approximation ratios exactly:
for any feasible $S$ and any optimal solution $S^\star$,
\[
\frac{F(S;T)}{F(S^\star;T)}
=
\frac{\operatorname{cov}(S)}{\operatorname{cov}(S^\star)}.
\]

Suppose, for contradiction, that there exists a polynomial-time algorithm for $(P_{\mathrm{cov}})$ achieving approximation factor $(1-1/e+\varepsilon)$ for some $\varepsilon>0$. Applying it to the reduced instance and using the identity above would produce, in polynomial time, a solution $S$ to Max-Coverage such that
\[
\operatorname{cov}(S)\ge (1-1/e+\varepsilon)\,\operatorname{cov}(S^\star).
\]
This contradicts the classical hardness-of-approximation barrier for Max-Coverage: for every $\varepsilon > 0$, no polynomial-time algorithm can achieve an approximation factor of $1 - 1/e + \varepsilon$ unless $P = NP$ (see, e.g., \cite{feige1998threshold}).

Therefore, no polynomial-time algorithm can approximate $(P_{\mathrm{cov}})$ within a factor strictly better than $1-1/e$ unless $P=NP$.
\end{proof}


\subsection{Proof of Theorem~\ref{thm:greedy}}
\label{app:thm4_5}

\begin{proof}
By Proposition~\ref{prop:submod}, the set function $F(\cdot;T)$ is monotone and submodular. Also, since $d_0=0$ and $S_0=\{0\}$ when $S=\emptyset$, we have
\[
F(\emptyset;T)=\frac1{|T|}\sum_{s\in T}\max_{i\in \{0\}}\langle d_i,s\rangle=0.
\]
Therefore, the classical Nemhauser--Wolsey--Fisher theorem for monotone submodular maximization under a cardinality constraint implies that the greedy algorithm returns a set $S_{\mathrm{greedy}}$ with
\[
F(S_{\mathrm{greedy}};T)\ge (1-1/e)\max_{|S|\le B}F(S;T).
\]
\end{proof}


\subsection{Proof of Corollary~\ref{cor:greedy-optimal}}
\label{app:cor4_6}

\begin{proof}
Immediate from Theorem~\ref{thm:hardness-approx} and Theorem~\ref{thm:greedy}.
\end{proof}


\section{AVERAGE-TO-WORST BRIDGE ON THE FINITE PROBE SET}
\label{app:average_worst}

This section formalizes a sufficient condition under which improving the surrogate objective
$F(S;T)$ also reduces the worst-case deficit $\max_{s \in T}\Delta_S(s)$ on the same probe set.

Recall the notation from Appendix C.3. For any index set $A \subseteq [N]_0$, define
\[
M_A(s) := \max_{i \in A}\langle d_i,s\rangle,
\qquad
\Delta_S(s) := M_{[N]_0}(s)-M_{S_0}(s).
\]
For a fixed nonempty finite set $T \subset \mathbb{R}^p$ and a subset $S \subseteq [N]$, define the average deficit
\[
\bar{\Delta}_S(T) := \frac{1}{|T|}\sum_{s \in T}\Delta_S(s).
\]
Using the definition of $F(\cdot;T)$, the surrogate gap satisfies the exact identity
\[
F([N];T)-F(S;T)=\bar{\Delta}_S(T).
\]

To convert this average quantity into a worst-case bound, define the local occupancy of the probe set under $d_D$:
\[
\mu_\eta(T) := \min_{s \in T}\left|\left\{t \in T : d_D(s,t)\le \eta\right\}\right|,
\qquad \eta \ge 0.
\]
Since $T\neq \varnothing$ and $d_D(s,s)=0$, we have $\mu_\eta(T)\ge 1$ for all $\eta\ge 0$.

\begin{proposition}[Average-to-worst bridge on $T$]
\label{prop:avg_to_worst_bridge}
Let $T \subset \mathbb{R}^p$ be nonempty and finite. Then, for any $S \subseteq [N]$ and any $\eta \ge 0$,
\[
\max_{s \in T}\Delta_S(s)
\;\le\;
\frac{|T|}{\mu_\eta(T)}\,\bar{\Delta}_S(T)+2\eta
\;=\;
\frac{|T|}{\mu_\eta(T)}\bigl(F([N];T)-F(S;T)\bigr)+2\eta.
\]
Equivalently,
\[
\max_{s \in T}\Delta_S(s)
\;\le\;
\inf_{\eta \ge 0}
\left\{
\frac{|T|}{\mu_\eta(T)}\,\bar{\Delta}_S(T)+2\eta
\right\}.
\]
\end{proposition}

\begin{proof}
Appendix C.3 shows that, for any index set $A \subseteq [N]_0$, the map $M_A(\cdot)$ is $1$-Lipschitz with respect to $d_D$:
\[
|M_A(s)-M_A(s')| \le d_D(s,s').
\]
Therefore, since $\Delta_S = M_{[N]_0}-M_{S_0}$,
\[
|\Delta_S(s)-\Delta_S(s')|
\le
|M_{[N]_0}(s)-M_{[N]_0}(s')|
+
|M_{S_0}(s)-M_{S_0}(s')|
\le
2\,d_D(s,s').
\]
Hence $\Delta_S(\cdot)$ is $2$-Lipschitz in $d_D$.

Let $s^\star \in \arg\max_{s \in T}\Delta_S(s)$ and define the $d_D$-ball on the probe set
\[
\mathcal{N}_\eta(s^\star) := \{t \in T : d_D(t,s^\star)\le \eta\}.
\]
For any $t \in \mathcal{N}_\eta(s^\star)$, the $2$-Lipschitz property gives
\[
\Delta_S(t)\ge \Delta_S(s^\star)-2\eta.
\]
Summing over $\mathcal{N}_\eta(s^\star)$ and using nonnegativity of $\Delta_S$ (since $S_0\subseteq [N]_0$ implies $M_{[N]_0}\ge M_{S_0}$ pointwise),
\[
|T|\,\bar{\Delta}_S(T)
=
\sum_{t \in T}\Delta_S(t)
\ge
\sum_{t \in \mathcal{N}_\eta(s^\star)}\Delta_S(t)
\ge
|\mathcal{N}_\eta(s^\star)|\bigl(\Delta_S(s^\star)-2\eta\bigr).
\]
By definition of $\mu_\eta(T)$, $|\mathcal{N}_\eta(s^\star)| \ge \mu_\eta(T)$, so
\[
|T|\,\bar{\Delta}_S(T)
\ge
\mu_\eta(T)\bigl(\max_{s \in T}\Delta_S(s)-2\eta\bigr).
\]
Rearranging yields
\[
\max_{s \in T}\Delta_S(s)
\le
\frac{|T|}{\mu_\eta(T)}\,\bar{\Delta}_S(T)+2\eta.
\]

Finally, identify $\bar{\Delta}_S(T)$ with the surrogate gap:
\[
\bar{\Delta}_S(T)
=
\frac{1}{|T|}\sum_{s \in T}\bigl(M_{[N]_0}(s)-M_{S_0}(s)\bigr)
=
F([N];T)-F(S;T).
\]
This proves the first display. The infimum form follows immediately by minimizing the right-hand side over $\eta\ge 0$.
\end{proof}

\begin{corollary}[Robust value-gap bound through the surrogate gap]
\label{cor:surrogate_to_value_gap}
Under the assumptions of Theorem~3.1, let $T$ be an $\varepsilon$-net of $q(X)$ in $d_D$. Then, for any $S \subseteq [N]$ and any $\eta \ge 0$,
\[
\min_x \Phi[N](x)-\min_x \Phi_S(x)
\;\le\;
r\left(
\frac{|T|}{\mu_\eta(T)}\bigl(F([N];T)-F(S;T)\bigr)
+
2\eta + 2\varepsilon
\right).
\]
Equivalently,
\[
\min_x \Phi[N](x)-\min_x \Phi_S(x)
\;\le\;
r\left(
\inf_{\eta\ge 0}
\left\{
\frac{|T|}{\mu_\eta(T)}\bigl(F([N];T)-F(S;T)\bigr)+2\eta
\right\}
+2\varepsilon
\right).
\]
\end{corollary}

\begin{proof}
Theorem~3.1 gives
\[
\min_x \Phi[N](x)-\min_x \Phi_S(x)
\le
r\bigl(\max_{s \in T}\Delta_S(s)+2\varepsilon\bigr).
\]
Apply Proposition~\ref{prop:avg_to_worst_bridge} to bound $\max_{s \in T}\Delta_S(s)$ and substitute.
\end{proof}

\paragraph{Practical interpretation.}
The factor $|T|/\mu_\eta(T)$ measures how densely the probe set covers itself under $d_D$ at resolution $\eta$. When $T$ is well distributed in the relevant directions, this factor is moderate, and improving the surrogate objective $F(S;T)$ reduces the worst-case certificate term on $T$, up to the additive resolution term $2\eta$. Conversely, if $T$ is highly nonuniform (small local occupancy near some directions), then average coverage can be less informative about the worst-case deficit; the bound makes this dependence explicit.

%% file: content/0_appendix_theo_2.tex
\section{Directional relevance in the robust LP template}
\label{sec:dir-rel-lp}

We first recall a minimal linear-programming instantiation of the robust template, and then use LP duality to isolate the directions of uncertainty that are \emph{value-relevant} for the full directional model. This yields a dual-based notion of relevance that will serve as a structural benchmark for any restricted family $U_S(r)$.

\paragraph{Minimal LP template.}
Let $X=\{x:Ax\le b\}$ and $f(x,u)=c(u)^\top x$ with $c(u)=c_0+Mu$. Then
\[
\Phi_S(x)=c_0^\top x + r\max_{i\in S_0}\ip{d_i}{M^\top x},
\]
and $\min_x \Phi_S(x)$ is an LP after introducing a variable to linearize the maximum. Pessimization over $[N]$ reduces to evaluating
\[
i^\star\in\argmax_{i\in [N]_0}\ip{d_i}{M^\top x}.
\]

\subsection{Setup and notation}
Assume a compact polyhedral feasible region
\[
X := \{x\in\mathbb{R}^n : Ax \le b\},
\]
and an affine objective coefficient map
\[
c(u) = c_0 + Mu,
\]
so that the nominal objective is $f(x,u)=c(u)^\top x$. The full directional robust objective is
\[
\Phi_{[N]}(x)=c_0^\top x + r\max_{i\in [N]_0}\ip{d_i}{M^\top x}.
\]

\subsection{Robust LP template and dual}

We adopt the following standing setup.

\begin{assumption}[Robust LP template]
\label{asp:lp-template}
Let $X:=\{x\in\R^n:Ax\le b\}$ be a nonempty, bounded polyhedron, with
$A\in\R^{m\times n}$ and $b\in\R^m$.
The uncertain cost is
\[
c(u)\ :=\ c_0+Mu,\qquad c_0\in\R^n,\ M\in\R^{n\times p},
\]
and the directional dictionary is
$\mathcal D=\{d_i\}_{i=1}^N\subset\R^p$, with $d_0:=0$ and
$[N]_0:=\{0,1,\dots,N\}$.
For $S\subseteq[N]$ and $r>0$,
\[
U_S(r)\ :=\ \Big\{u=\textstyle\sum_{i\in S}\alpha_i d_i:\ \alpha_i\ge 0,\ \sum_{i\in S}\alpha_i\le r\Big\}.
\]
The \emph{full} model corresponds to $S=[N]$.
\end{assumption}

The robust value of the full model is
\[
z_{[N]}^\star
\ :=\
\min_{x\in X}\ \sup_{u\in U_{[N]}(r)} (c_0+Mu)^\top x.
\]
Using the support function of $U_{[N]}(r)$ and introducing an epigraph variable
$t\in\R$ we obtain the equivalent LP
\begin{equation}
\label{eq:full-primal}
\begin{aligned}
\text{(P$[N]$)}\qquad
z_{[N]}^\star
\ =\ \min_{x,t}\quad & c_0^\top x + r t\\
\text{s.t.}\quad & Ax \le b,\\
& (Md_i)^\top x \le t,\quad \forall i\in[N]_0.
\end{aligned}
\end{equation}

We now write its dual in a form that exposes the role of the dictionary.

\begin{proposition}[Dual of the full robust LP]
\label{prop:full-dual}
Under Assumption~\ref{asp:lp-template}, the dual of \emph{(P$[N]$)} is
\begin{equation}
\label{eq:full-dual}
\begin{aligned}
\text{(D$[N]$)}\qquad
z_{[N]}^\star
\ =\ \max_{\lambda,\mu}\quad & -\,b^\top \lambda\\
\text{s.t.}\quad
& A^\top \lambda\ =\ -\,c_0 - M\Big(\textstyle\sum_{i\in[N]_0}\mu_i d_i\Big),\\
& \sum_{i\in[N]_0} \mu_i\ =\ r,\\
& \lambda\ \ge\ 0,\quad \mu_i\ \ge\ 0,\ \forall i\in[N]_0.
\end{aligned}
\end{equation}
Strong duality holds: \emph{(P$[N]$)} and \emph{(D$[N]$)} have the same finite
optimal value $z_{[N]}^\star$.
\end{proposition}

\begin{proof}
Assign dual variables $\lambda\in\R^m_+$ to $Ax\le b$ and
$\mu_i\ge 0$ to $(Md_i)^\top x - t \le 0$ for $i\in[N]_0$.
The Lagrangian of \eqref{eq:full-primal} is
\[
L(x,t;\lambda,\mu)
=\big(c_0 + A^\top\lambda + M\textstyle\sum_{i\in[N]_0}\mu_i d_i\big)^\top x
+\big(r-\sum_{i\in[N]_0}\mu_i\big)t - b^\top\lambda.
\]
For $\inf_{x,t} L(x,t;\lambda,\mu)$ to be finite we need both coefficients of
$x$ and $t$ to vanish. This yields the two equalities in
\eqref{eq:full-dual}; under them, $L(x,t;\lambda,\mu)\equiv -b^\top\lambda$.
Maximizing over $(\lambda,\mu)$ subject to feasibility gives \emph{(D$[N]$)}.
Strong duality follows from standard LP duality under boundedness and
feasibility of $X$.
\end{proof}

\subsection{Value-relevant directions and the relevant face}

Optimal dual solutions reveal which dictionary atoms are actually used in the
worst-case cost at the robust optimum.

\begin{definition}[Value-relevant indices and face]
Let $(x^\star,t^\star)$ and $(\lambda^\star,\mu^\star)$ be optimal solutions of
\emph{(P$[N]$)} and \emph{(D$[N]$)}, respectively. Define the set of
\emph{value-relevant indices}
\[
I_{\mathrm{rel}}\ :=\ \{i\in[N] : \mu_i^\star > 0\},
\]
and the corresponding \emph{relevant face}
\[
C_{\mathrm{rel}}\ :=\ \conv\!\big(\{d_i : i\in I_{\mathrm{rel}}\}\cup\{0\}\big)
\ \subseteq\ \conv\!\big(\{d_i : i\in[N]\}\cup\{0\}\big).
\]
We also define the associated worst-case direction
\[
u^\star\ :=\ \sum_{i\in[N]_0}\mu_i^\star d_i\in U_{[N]}(r),
\qquad
v^\star\ :=\ \frac{1}{r}\,u^\star\in C_{\mathrm{rel}}.
\]
\end{definition}

By complementary slackness, every index $i\in I_{\mathrm{rel}}$ corresponds to a
binding constraint $(Md_i)^\top x^\star=t^\star$ in \emph{(P$[N]$)}. The vector
$u^\star$ is a worst-case perturbation at $x^\star$, and $v^\star$ lies in an
exposed face of $\conv(\mathcal D\cup\{0\})$ determined by $M^\top x^\star$.

\subsection{Directional sufficiency via the dual}

For $S\subseteq[N]$, define the restricted uncertainty set $U_S(r)$ as in
Assumption~\ref{asp:lp-template}, and denote by $z_S^\star$ the corresponding
robust value:
\[
z_S^\star\ :=\ \min_{x\in X}\ \sup_{u\in U_S(r)} (c_0+Mu)^\top x.
\]

The following theorem gives a dual-based notion of \emph{directional sufficiency}
for the full robust value.

\begin{theorem}[Global directional sufficiency via the dual]
\label{thm:dir-suff-global}
Adopt Assumption~\ref{asp:lp-template} and let
$(\lambda^\star,\mu^\star)$ be any optimal solution of \emph{(D$[N]$)},
with associated value-relevant set $I_{\mathrm{rel}}$ and face
$C_{\mathrm{rel}}=\conv(\{d_i : i\in I_{\mathrm{rel}}\}\cup\{0\})$.
Let $S\subseteq[N]$ and suppose that
\begin{equation}
\label{eq:crel-in-S-global}
C_{\mathrm{rel}}\ \subseteq\ \conv\{d_i : i\in S_0\}.
\end{equation}
Then the robust values of the full and restricted models coincide:
\[
z_S^\star\ =\ z_{[N]}^\star.
\]
\end{theorem}

\begin{proof}
Let $u^\star:=\sum_{i\in[N]_0}\mu_i^\star d_i\in U_{[N]}(r)$ and
$v^\star:=u^\star/r\in C_{\mathrm{rel}}$.
By \eqref{eq:crel-in-S-global}, there exist coefficients
$\{\alpha_j\}_{j\in S_0}$ with $\alpha_j\ge 0$ and
$\sum_{j\in S_0}\alpha_j=1$ such that
\[
v^\star=\sum_{j\in S_0}\alpha_j d_j,
\qquad\text{hence}\qquad
u^\star
=\ r v^\star
=\ \sum_{j\in S_0}(r\alpha_j)d_j.
\]
Define coefficients $\widetilde\mu_j:=r\alpha_j$ for $j\in S_0$. Then
$\widetilde\mu_j\ge 0$ and $\sum_{j\in S_0}\widetilde\mu_j=r$, and
\[
c_0 + M\sum_{j\in S_0}\widetilde\mu_j d_j
=\ c_0 + Mu^\star
=\ c_0 + M\sum_{i\in[N]_0}\mu_i^\star d_i
=\ -A^\top\lambda^\star,
\]
where the last equality uses feasibility of $(\lambda^\star,\mu^\star)$ in
\emph{(D$[N]$)}.
Thus $(\lambda^\star,\widetilde\mu)$ is feasible for the dual of the
restricted model \emph{(D$[S]$)}, with objective value
$-b^\top\lambda^\star=z_{[N]}^\star$.
By strong duality for \emph{(P$[S]$)}--\emph{(D$[S]$)},
\[
z_S^\star
=\max\{-b^\top\lambda:\ (\lambda,\mu)\ \text{feasible in \emph{(D$[S]$)}}\}
\ \ge\ -b^\top\lambda^\star
\ =\ z_{[N]}^\star.
\]
On the other hand, $U_S(r)\subseteq U_{[N]}(r)$ implies monotonicity of the
robust value in the uncertainty set, whence $z_S^\star\le z_{[N]}^\star$.
Combining the two inequalities yields $z_S^\star=z_{[N]}^\star$.
\end{proof}

\subsection{Instance-wise no-free-lunch for insufficient subsets}

We now show that if a subset $S$ fails to represent the relevant convex
geometry of the dictionary, then there is no uniform guarantee relating the
restricted and full robust values: one can construct instances for which the
value gap is arbitrarily large.

\begin{theorem}[Instance-wise no-free-lunch for insufficient subsets]
\label{thm:nofreelunch-instance}
Let $\mathcal D=\{d_i\}_{i=1}^N\subset\R^p$ be a finite dictionary and
let $S\subset[N]$ be such that
\[
\conv\{d_i : i\in S\}\ \subsetneq\ \conv\{d_i : i\in[N]\}.
\]
Then there exist a dimension $n\in\mathbb{N}$, a nonempty bounded polyhedron
$X\subset\R^n$ and data $(A,b,c_0,M,r)$ for the robust LP template of
Assumption~\ref{asp:lp-template} (with this $X$ and $\mathcal D$) such that,
for the corresponding full and restricted robust values $z_{[N]}^\star$ and
$z_S^\star$, the following holds: for every $K>0$ one can choose
$(A,b,c_0,M,r)$ with
\[
z_{[N]}^\star\ -\ z_S^\star\ \ge\ K.
\]
In particular, the family of restricted models based on $U_S(r)$ cannot provide
a uniform approximation of the full directional model across all robust LP
instances built on the same dictionary~$\mathcal D$.
\end{theorem}

\begin{proof}
Since $\conv\{d_i : i\in S\}\subsetneq\conv\{d_i : i\in[N]\}$, the polytope
$\conv\{d_i : i\in[N]\}$ has at least one extreme point that does not belong to
$\conv\{d_i : i\in S\}$. Let $d_{i_0}$ be such an extreme point. By the
separating hyperplane theorem, there exists $q\in\R^p$ such that
\[
\langle d_{i_0},q\rangle\ >\ \max_{y\in \conv\{d_i : i\in S\}}\langle y,q\rangle
\ =\ \max_{j\in S}\langle d_j,q\rangle.
\]
In particular,
\[
\Delta\ :=\ \max_{i\in[N]}\langle d_i,q\rangle\ -\ \max_{j\in S}\langle d_j,q\rangle\ >\ 0.
\]

We construct an instance in dimension $n=1$. Let
\[
X\ :=\ \{x\in\R : x=1\},
\]
which is a nonempty bounded polyhedron (it can be written as
$\{x: x\le 1,\ -x\le -1\}$).
Take $c_0:=0\in\R$ and define $M\in\R^{1\times p}$ by $M:=\begin{bmatrix}q^\top\end{bmatrix}$.
For any $u\in\R^p$ and $x\in X$,
\[
(c_0+Mu)^\top x\ =\ (Mu)\,x\ =\ q^\top u.
\]

For radius $r>0$, the robust values of the full and restricted models are
\[
z_{[N]}^\star
=\sup_{u\in U_{[N]}(r)}q^\top u
= r\,\max_{i\in[N]_0}\langle d_i,q\rangle,
\qquad
z_S^\star
=\sup_{u\in U_S(r)}q^\top u
= r\,\max_{j\in S_0}\langle d_j,q\rangle,
\]
where $S_0:=S\cup\{0\}$ and $[N]_0:=\{0,1,\dots,N\}$.
By definition of $\Delta$ and the inclusion of $0$ in both index sets, we have
\[
\max_{i\in[N]_0}\langle d_i,q\rangle
-\max_{j\in S_0}\langle d_j,q\rangle
\ \ge\
\max_{i\in[N]}\langle d_i,q\rangle
-\max_{j\in S}\langle d_j,q\rangle
=\ \Delta\ > 0.
\]
Therefore
\[
z_{[N]}^\star - z_S^\star
= r\Big(\max_{i\in[N]_0}\langle d_i,q\rangle
-\max_{j\in S_0}\langle d_j,q\rangle\Big)
\ \ge\ r\,\Delta.
\]
Given any prescribed $K>0$, choosing $r:=K/\Delta$ yields
$z_{[N]}^\star - z_S^\star \ge K$, as claimed.
\end{proof}

\section{Greedy with inexact oracle/selection and value certificate}
\label{sec:greedy-inexact}

In practice, the marginal gains of $F(\cdot;\mathcal T)$ (see \eqref{eq:F}) are computed with error
and the selection rule can be suboptimal. We model this per iteration.

\begin{assumption}[Per-iteration inexact selection]\label{asp:inexact}
At iteration $t$, let $\Delta F_t^{\max}$ be the best possible marginal gain of $F(\cdot;\mathcal T)$
among all indices $i\in [N]\setminus S_{t-1}$, that is,
\[
\Delta F_t^{\max}\ :=\ \max_{i\in [N]\setminus S_{t-1}}\ \big[F(S_{t-1}\cup\{i\};\mathcal T)-F(S_{t-1};\mathcal T)\big].
\]
We assume the chosen index $i_t$ satisfies
\[
\Delta F_t(\text{chosen})\ \ge\ (1-\rho)\,\Delta F_t^{\max}\ -\ \varepsilon_t,
\]
with $\rho\in[0,1)$ (multiplicative suboptimality) and $\varepsilon_t\ge 0$ (additive error).
\end{assumption}

\begin{theorem}[Inexact greedy guarantee for $F$]\label{thm:greedy-inexact}
Under submodularity and monotonicity of $F(\cdot;\mathcal T)$ and Assumption~\ref{asp:inexact}, after $B$ steps,
\[
F(S_{\rm aprx};\mathcal T)\ \ge\ 
\Big(1-\Big(1-\tfrac{1-\rho}{B}\Big)^{\!B}\Big)\,F(S^\star;\mathcal T)\ -\ 
\sum_{t=1}^B \Big(1-\tfrac{1-\rho}{B}\Big)^{\!B-t}\,\varepsilon_t,
\]
where $S^\star$ maximizes $F(\cdot;\mathcal T)$ with $|S^\star|\le B$. In particular,
\[
\Big(1-\tfrac{1-\rho}{B}\Big)^{\!B}\ \le\ e^{-(1-\rho)}
\quad\Rightarrow\quad
F(S_{\rm aprx};\mathcal T)\ \ge\ \big(1-e^{-(1-\rho)}\big)\,F(S^\star;\mathcal T)\ -\ 
\sum_{t=1}^B e^{-\frac{1-\rho}{B}(B-t)}\,\varepsilon_t.
\]
\end{theorem}

\begin{proof}
Let $S_t$ be the subset selected after $t$ iterations, with $S_0=\varnothing$, and let
\[
G_t := F(S^\star;\mathcal T)-F(S_t;\mathcal T).
\]
We write $F(\cdot)$ for $F(\cdot;\mathcal T)$ to simplify notation.

By monotonicity and submodularity, the standard greedy lower bound gives
\[
\Delta F_t^{\max}\ \ge\ \frac{1}{B}\big(F(S^\star)-F(S_{t-1})\big)
\;=\;\frac{1}{B}G_{t-1}.
\]
(Indeed, adding the elements of $S^\star\setminus S_{t-1}$ one by one and using submodularity yields
$F(S^\star)-F(S_{t-1})\le B\,\Delta F_t^{\max}$ since $|S^\star|\le B$.)

By Assumption~\ref{asp:inexact},
\[
F(S_t)-F(S_{t-1})
\;=\;\Delta F_t(\text{chosen})
\;\ge\; (1-\rho)\,\Delta F_t^{\max}-\varepsilon_t
\;\ge\; \frac{1-\rho}{B}G_{t-1}-\varepsilon_t.
\]
Rearranging,
\[
G_t
=F(S^\star)-F(S_t)
\le \Big(1-\frac{1-\rho}{B}\Big)G_{t-1}+\varepsilon_t.
\]

Set
\[
a:=1-\frac{1-\rho}{B}\in[0,1).
\]
Then the recurrence is
\[
G_t\le a\,G_{t-1}+\varepsilon_t.
\]
Unrolling for $t=1,\dots,B$ gives
\[
G_B\le a^B G_0+\sum_{t=1}^B a^{B-t}\varepsilon_t.
\]
Since $S_0=\varnothing$ and $F(\varnothing;\mathcal T)=0$ (because $d_0=0$ is included), we have
\[
G_0=F(S^\star;\mathcal T)-F(\varnothing;\mathcal T)=F(S^\star;\mathcal T).
\]
Therefore,
\[
F(S_B;\mathcal T)
=F(S^\star;\mathcal T)-G_B
\ge (1-a^B)F(S^\star;\mathcal T)-\sum_{t=1}^B a^{B-t}\varepsilon_t.
\]
Since $S_{\rm aprx}=S_B$, this proves the first claim.

For the exponential form, use $(1-x)\le e^{-x}$ with $x=(1-\rho)/B$:
\[
a^B=\Big(1-\frac{1-\rho}{B}\Big)^B\le e^{-(1-\rho)},
\qquad
a^{B-t}\le e^{-\frac{1-\rho}{B}(B-t)}.
\]
Substituting these bounds yields the second inequality.
\end{proof}

\begin{remark}[Mapping oracle errors to $(\rho,\varepsilon_t)$]\label{rmk:oracle-to-rhoeps}
If at iteration $t$ the marginal gains are evaluated with a \emph{uniform} additive error $\delta_t$
\emph{both} for $F(S_{t-1}\cup\{i\})$ \emph{and} for $F(S_{t-1})$, then $\rho=0$ and $\varepsilon_t\le 2\delta_t$
(the gain is a difference of two quantities, each with error $\le \delta_t$).
If moreover only a multiplicative factor $(1-\tilde\rho)$ with respect to the best gain is guaranteed,
one may take $\rho=\tilde\rho$.
\end{remark}

\paragraph{Robust value certificate with inexact selection.}
Combining the pointwise value-gap identity \eqref{eq:gap-ptwise}--\eqref{eq:gap-optvals} with a finite
family $\mathcal T$ (Section~\ref{sec:enet-bridge}), we obtain the \textbf{deterministic certificate}
\begin{equation}\label{eq:certificado-inexacto}
\min_{x} \Phi_{[N]}(x)\ -\ \min_{x} \Phi_{S_{\rm aprx}}(x)
\ \le\ r\Big(\max_{s\in\mathcal T}\Delta_{S_{\rm aprx}}(s)\ +\ 2\varepsilon\Big),
\end{equation}
where $\varepsilon$ is the grid granularity over $q(X)$ (Theorem~\ref{thm:enet-bridge}).
The right-hand side is computable at run time.

\paragraph{Certified stopping rule.}
Fix a threshold $\tau>0$; stop when
\[
\max_{s\in\mathcal T}\Delta_{S}(s)\ \le\ \frac{\tau}{2r}
\qquad\text{and}\qquad
\varepsilon\ \le\ \frac{\tau}{4r},
\]
and report the certificate \(\min\Phi_{[N]}-\min\Phi_{S}\le \tau\).

%% file: content/0_appendix_exp.tex
\section{Experimental details}
\label{app:exp-details}

This appendix provides implementation details for the experiments in Sec.~\ref{sec:computational_exp}, covering out-of-sample calibration, frozen feature classification, and scaling and runtime metrics.

\paragraph{Hardware and implementation details.}
\label{app:runtime-details}
All methods were implemented in \texttt{Python 3.10.12} and experiments ran on Ubuntu Linux with an Intel\textregistered{} Xeon\textregistered{} Gold 6230R processor (104 logical cores, 2.10\,GHz) and 187\,GiB of memory. Wall clock runtime is reported.

\input{content/0_appendix_out}

\subsection{Synthetic geometric validation}
\label{app:teacher-details}

A synthetic two dimensional dictionary with $N=2000$ atoms and budgets $B\in\{2,4,8,16,32\}$ is utilized. Support approximation is evaluated along $x_{\mathrm{fixed}}=(1,0)$. A calibration set of size $m_{\mathrm{cal}}=1000$ is employed with $(\alpha,\delta)=(0.05,10^{-5})$.

The dictionary is generated once and reused. It mixes structured directions, isotropic noise, and decoy directions orthogonal to $x_{\mathrm{fixed}}$.

For any subset $S$, the following is reported:
\[
z_S(r)=\sup_{u\in U_S(r)}\langle u,x_{\mathrm{fixed}}\rangle
=r\max\!\Bigl(0,\max_{i\in S}\langle d_i,x_{\mathrm{fixed}}\rangle\Bigr),
\]
and compared to $z_{[N]}(r)$. An ellipsoidal baseline fitted to the calibration sample is included.

\subsubsection{Surrogate versus certificate alignment in the synthetic validation}
\label{app:surrogate_vs_certificate}

This subsection quantifies the relationship between the surrogate selection objective and the reported worst deficit certificate proxy in the synthetic geometric validation. The selection step uses the average coverage objective
\[
F(S;T_{\mathrm{eval}}),
\]
while the reporting includes the worst deficit proxy
\[
G(S;T_{\mathrm{eval}})
:= r_{\mathrm{eval}} \max_{s \in T_{\mathrm{eval}}} \Delta_S(s).
\]
A direct empirical comparison between these quantities clarifies when improvements in the surrogate objective align with improvements in the certificate proxy.

For each method and budget, the following normalized metrics are evaluated:
\[
F_{\mathrm{ratio}}(S)
:= \frac{F(S;T_{\mathrm{eval}})}{F([N];T_{\mathrm{eval}})},
\qquad
G_{\mathrm{ratio}}(S)
:= \frac{G(S;T_{\mathrm{eval}})}{G(\varnothing;T_{\mathrm{eval}})}.
\]
The correlation is computed between $F_{\mathrm{ratio}}$ and $-G_{\mathrm{ratio}}$, so larger values on both axes correspond to better performance.

The statistics in this subsection use all available budgets exported by the synthetic experiment. Table~\ref{tab:surrogate_certificate_summary} reports the aggregate statistics. The Spearman correlation indicates a monotone alignment between the surrogate objective and the certificate proxy in this setup. The Greedy Coverage and Greedy MaxGap comparison shows divergence only at budgets $B \in \{1,3,4\}$. This localizes the mismatch to the micro budget regime.

\begin{table}[!h]
\centering
\caption{Summary of surrogate versus certificate alignment in the synthetic geometric validation.}
\label{tab:surrogate_certificate_summary}
\small
\begin{tabular}{l c}
\hline
Quantity & Value \\
\hline
Number of analyzed points & 124 \\
Pearson corr$(F_{\mathrm{ratio}}, -G_{\mathrm{ratio}})$ & 0.8554 \\
Spearman corr$(F_{\mathrm{ratio}}, -G_{\mathrm{ratio}})$ & 0.9617 \\
Budgets with GC versus GM divergence & 3 \\
Divergence budgets & $\{1,3,4\}$ \\
\hline
\end{tabular}
\end{table}

Figure~\ref{fig:surrogate_certificate_relation} plots the normalized surrogate objective against the normalized certificate proxy across methods and budgets, together with a direct Greedy Coverage versus Greedy MaxGap difference diagnostic. The first panel shows a monotone trend. The second panel reports two normalized differences for each budget:
\[
\Delta F := F_{\mathrm{ratio}}^{\mathrm{GC}} - F_{\mathrm{ratio}}^{\mathrm{GM}},
\qquad
\Delta G_{\mathrm{score}} := G_{\mathrm{ratio}}^{\mathrm{GM}} - G_{\mathrm{ratio}}^{\mathrm{GC}},
\]
where positive $\Delta G_{\mathrm{score}}$ means Greedy Coverage achieves a smaller certificate proxy. A sign disagreement between $\Delta F$ and $\Delta G_{\mathrm{score}}$ indicates a surrogate versus certificate divergence.

\begin{figure}[!h]
    \centering
    \includegraphics[width=0.8\linewidth]{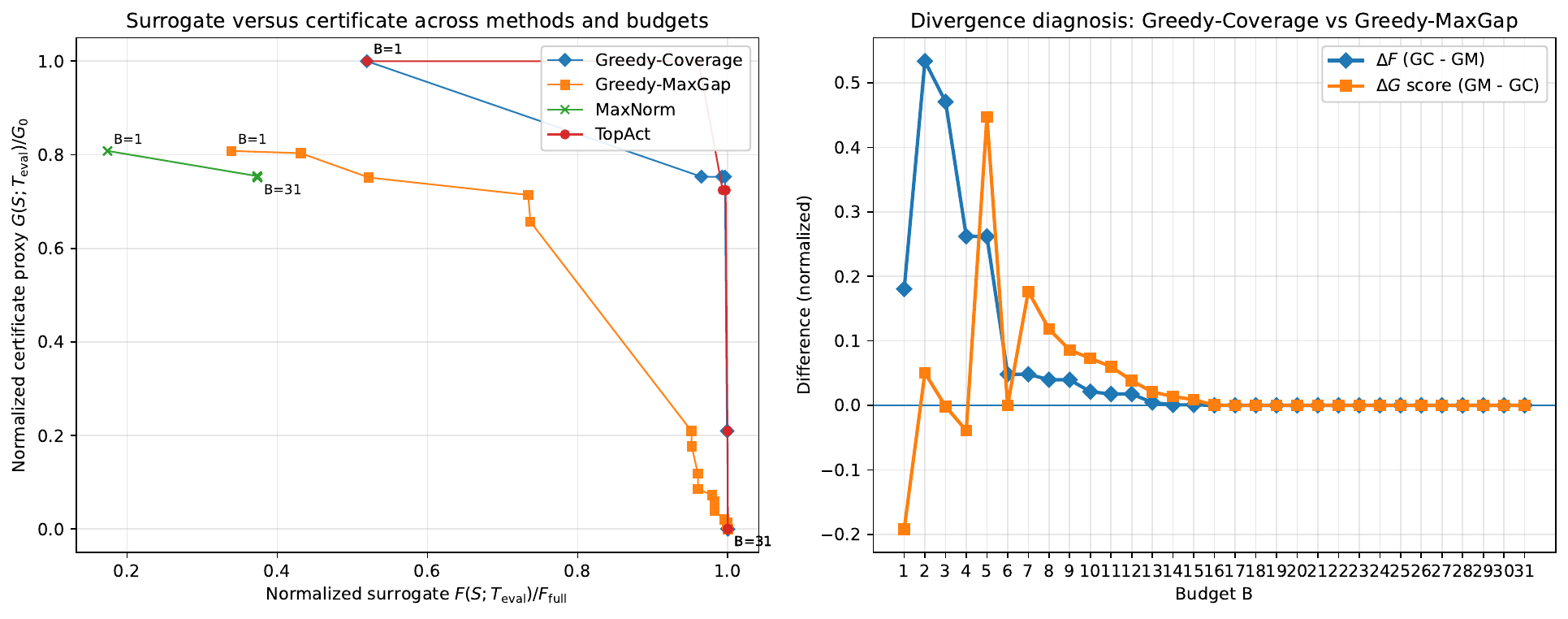}
    \caption{Surrogate versus certificate analysis in the synthetic geometric validation. Left: normalized surrogate objective $F_{\mathrm{ratio}}$ versus normalized certificate proxy $G_{\mathrm{ratio}}$ across methods and budgets. Lower values of $G_{\mathrm{ratio}}$ are better. Right: direct comparison between Greedy Coverage and Greedy MaxGap through $\Delta F$ and $\Delta G_{\mathrm{score}}$ as functions of the budget.}
    \label{fig:surrogate_certificate_relation}
\end{figure}

Table~\ref{tab:surrogate_certificate_gcgm} reports representative budget values for the Greedy Coverage versus Greedy MaxGap comparison. Divergence is confined to the smallest budgets, while the two criteria align after that regime.

\begin{table}[!h]
\centering
\caption{Representative Greedy Coverage versus Greedy MaxGap comparisons. Positive $\Delta F$ favors Greedy Coverage on the surrogate objective. Positive $\Delta G_{\mathrm{score}}$ favors Greedy Coverage on the certificate proxy.}
\label{tab:surrogate_certificate_gcgm}
\scriptsize
\begin{tabular}{c c c c c c c}
\hline
$B$ & $\Delta F$ & $\Delta G_{\mathrm{score}}$ & GC $F_{\mathrm{ratio}}$ & GC $G_{\mathrm{ratio}}$ & GM $F_{\mathrm{ratio}}$ & GM $G_{\mathrm{ratio}}$ \\
\hline
$1$ & $0.1803$ & $-0.1918$ & $0.5193$ & $1.0000$ & $0.3390$ & $0.8082$ \\
$3$ & $0.4705$ & $-0.0017$ & $0.9924$ & $0.7531$ & $0.5219$ & $0.7513$ \\
$4$ & $0.2620$ & $-0.0390$ & $0.9962$ & $0.7531$ & $0.7342$ & $0.7140$ \\
$5$ & $0.2618$ & $0.4472$ & $0.9989$ & $0.2097$ & $0.7371$ & $0.6568$ \\
$9$ & $0.0396$ & $0.0857$ & $1.0000$ & $0.0000$ & $0.9604$ & $0.0857$ \\
$19$ & $0.0000$ & $0.0000$ & $1.0000$ & $0.0000$ & $1.0000$ & $0.0000$ \\
$29$ & $0.0000$ & $0.0000$ & $1.0000$ & $0.0000$ & $1.0000$ & $0.0000$ \\
\hline
\end{tabular}
\end{table}

Table~\ref{tab:surrogate_certificate_method_paths} summarizes the start and end points of each method trajectory. The normalized trajectories support the same conclusion: the surrogate objective acts as a guide for the certificate proxy beyond the smallest budgets in this synthetic setup.

\begin{table}[!h]
\centering
\caption{Method trajectory endpoints in normalized surrogate and certificate coordinates.}
\label{tab:surrogate_certificate_method_paths}
\small
\begin{tabular}{l c c}
\hline
Method & $F_{\mathrm{ratio}}$ (start $\to$ end) & $G_{\mathrm{ratio}}$ (start $\to$ end) \\
\hline
Greedy Coverage & $0.5193 \to 1.0000$ & $1.0000 \to 0.0000$ \\
Greedy MaxGap & $0.3390 \to 1.0000$ & $0.8082 \to 0.0000$ \\
MaxNorm & $0.1736 \to 0.3737$ & $0.8082 \to 0.7514$ \\
TopAct & $0.5193 \to 1.0000$ & $1.0000 \to 0.0000$ \\
\hline
\end{tabular}
\end{table}

In the synthetic geometric validation, the surrogate objective $F$ and the certificate proxy $G$ align well overall, with divergence concentrated in the micro budget regime. This supports using the submodular surrogate for selection while preserving the worst deficit metric as the reporting and diagnostic quantity.

\subsection{Operationalizing the \texorpdfstring{$\varepsilon$}{epsilon} term in Theorem 3.1}
\label{app:epsilon_practicality}

This section reports an empirical evaluation of the discretization term in Theorem 3.1, quantifying the approximation gap induced by replacing the direction set $q(X)$ with a finite set $T$.

Let
$$ d_D(s,t) = \max_i \left| \langle d_i, s-t \rangle \right| $$
and let $T_{\mathrm{probe}}$ denote an independent probe pool of normalized directions. The empirical discretization error is defined as
$$ \widehat{\varepsilon}(T) = \max_{s \in T_{\mathrm{probe}}} \min_{t \in T} d_D(s,t). $$
This quantity is the empirical counterpart of the $\varepsilon$ net radius in Theorem 3.1.

In the synthetic two dimensional experiment, the same finite set $T$ is used for selection and reporting, with $|T| = 801$, and the independent probe set has $|T_{\mathrm{probe}}| = 4000$. The estimate is:
$$ \widehat{\varepsilon}(T) = 1.824282, \qquad 2\widehat{\varepsilon}(T) = 3.648563. $$

\begin{figure}[!h]
    \centering
    \includegraphics[width=0.45\linewidth]{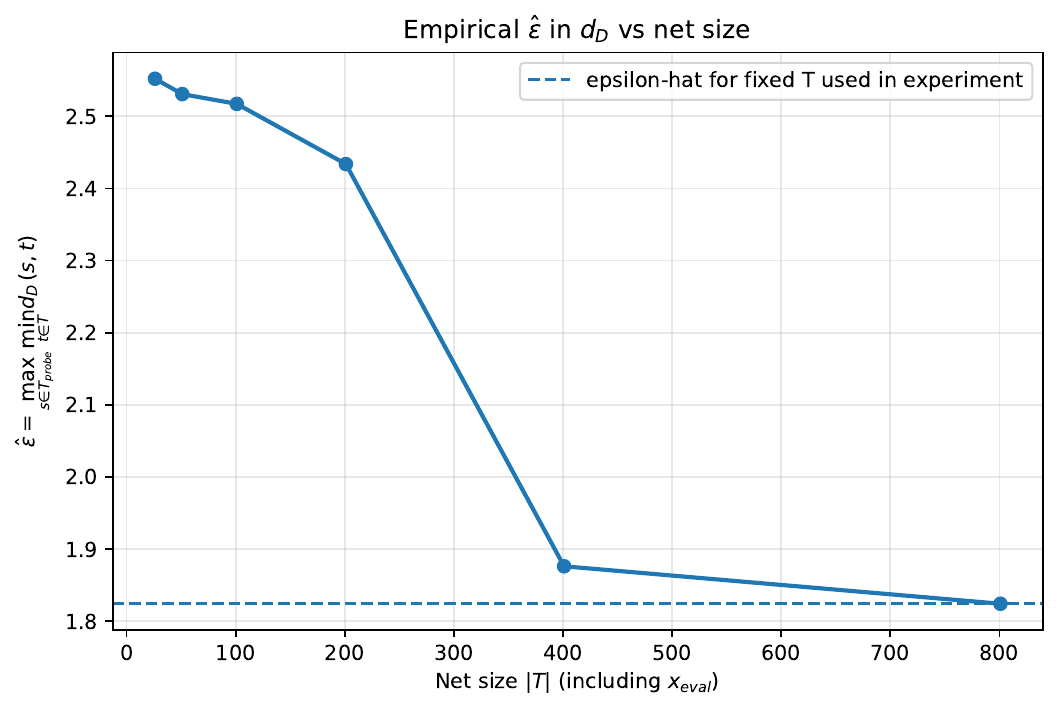}
    \caption{
    Empirical discretization error $\widehat{\varepsilon}(T)$ in the metric $d_D$ as a function of the net size $|T|$. The curve is computed from nested prefixes of the direction pool and an independent probe set $T_{\mathrm{probe}}$. The estimated curve decreases monotonically over the reported refinement path.
    }
    \label{fig:epsilon_practicality}
\end{figure}

Table~\ref{tab:epsilon_summary} summarizes the statistics. The mean and the 95th percentile of the nearest net distance are small relative to the supremum $\widehat{\varepsilon}(T)$, indicating the worst case discretization term is driven by a small subset of probe directions.

\begin{table}[!h]
\centering
\caption{Summary of the empirical $\widehat{\varepsilon}(T)$ estimate in the synthetic two dimensional experiment.}
\label{tab:epsilon_summary}
\begin{tabular}{lc}
\toprule
Quantity & Value \\
\midrule
$|T|$ used in selection and reporting & 801 \\
$|T_{\mathrm{probe}}|$ & 4000 \\
Number of metric atoms used for $d_D$ & 8 \\
$\widehat{\varepsilon}(T)$ & 1.824282 \\
$2\widehat{\varepsilon}(T)$ & 3.648563 \\
Mean nearest net distance on $T_{\mathrm{probe}}$ & 0.013876 \\
95th percentile nearest net distance on $T_{\mathrm{probe}}$ & 0.027180 \\
\bottomrule
\end{tabular}
\end{table}

Table~\ref{tab:epsilon_curve_values} reports the refinement path. The empirical estimate decreases from 2.551961 to 1.824282 as $|T|$ increases from 26 to 801, representing a reduction of 28.51\%.

\begin{table}[!h]
\centering
\caption{Empirical $\widehat{\varepsilon}(T)$ along the net refinement path.}
\label{tab:epsilon_curve_values}
\begin{tabular}{cc}
\toprule
$|T|$ & $\widehat{\varepsilon}(T)$ \\
\midrule
26 & 2.551961 \\
51 & 2.530695 \\
101 & 2.517097 \\
201 & 2.433706 \\
401 & 1.876200 \\
801 & 1.824282 \\
\bottomrule
\end{tabular}
\end{table}

Theorem 3.1 bounds the robust value gap through the term
$$ \max_{t \in T} \Delta_S(t) + 2\varepsilon. $$
Using the empirical estimate $\widehat{\varepsilon}(T)$, this section reports the practical decomposition
$$ \max_{t \in T} \Delta_S(t) + 2\widehat{\varepsilon}(T). $$
Figure~\ref{fig:theorem31_decomposition} shows this decomposition for Greedy Coverage and Greedy MaxGap.

\begin{figure}[!h]
    \centering
    \includegraphics[width=0.5\linewidth]{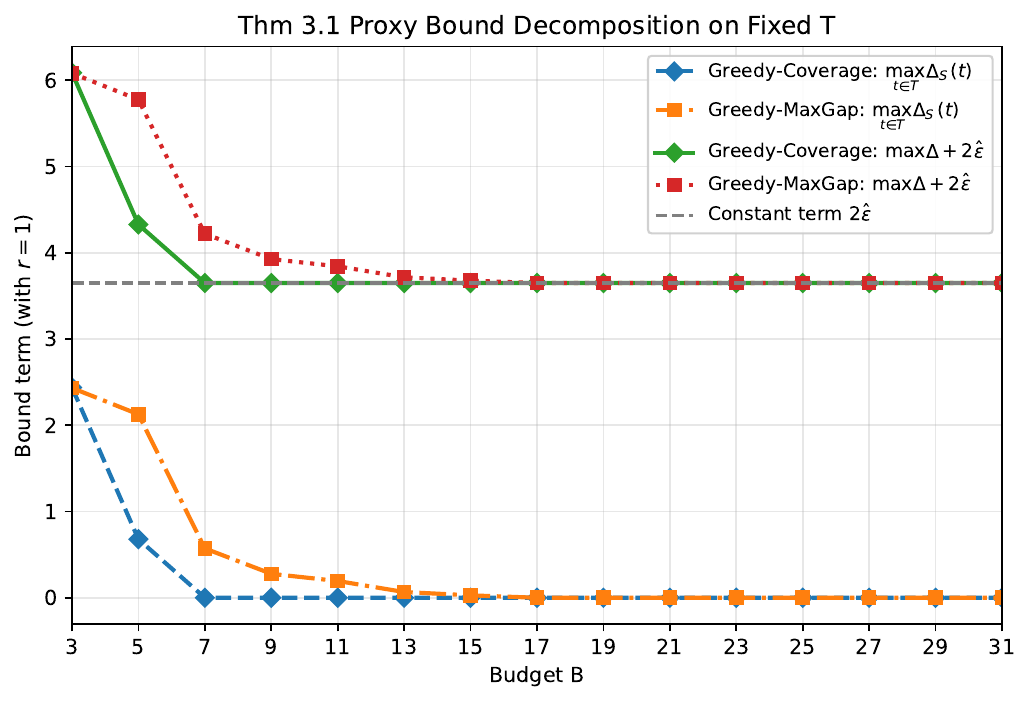}
    \caption{
    Empirical decomposition of the Theorem 3.1 bound term on the finite set $T$. The curves show $\max_{t \in T}\Delta_S(t)$ and $\max_{t \in T}\Delta_S(t) + 2\widehat{\varepsilon}(T)$ for Greedy Coverage and Greedy MaxGap. The horizontal line corresponds to the constant discretization contribution $2\widehat{\varepsilon}(T)$.
    }
    \label{fig:theorem31_decomposition}
\end{figure}

The decomposition reveals two regimes. For small budgets, the reduction in the bound results from the decrease in $\max_{t \in T}\Delta_S(t)$. For moderate budgets, the alignment deficit becomes small, and the constant term $2\widehat{\varepsilon}(T)$ dominates. In this range, increasing the atom budget $B$ has minimal effect on the certificate, and the primary method to tighten the bound is the refinement of the direction set $T$.

For reference, Greedy Coverage reaches $\max_{t \in T}\Delta_S(t)=0$ at $B=9$ in this experiment, while Greedy MaxGap reaches $\max_{t \in T}\Delta_S(t)=0$ at $B=17$.

\begin{table}[!h]
\centering
\caption{Representative values of the practical Theorem 3.1 decomposition for Greedy Coverage and Greedy MaxGap.}
\label{tab:gc_mg_certificate}
\begin{tabular}{cccccc}
\toprule
$B$ &
\multicolumn{2}{c}{Greedy Coverage} &
\multicolumn{2}{c}{Greedy MaxGap} &
$2\widehat{\varepsilon}(T)$ \\
\cmidrule(r){2-3}\cmidrule(r){4-5}
& $\max_{t\in T}\Delta_S(t)$ & $+\;2\widehat{\varepsilon}(T)$ &
$\max_{t\in T}\Delta_S(t)$ & $+\;2\widehat{\varepsilon}(T)$ &
\\
\midrule
3 & 2.438163 & 6.086726 & 2.432507 & 6.081070 & 3.648563 \\
5 & 0.678860 & 4.327424 & 2.126637 & 5.775201 & 3.648563 \\
9 & 0.000000 & 3.648563 & 0.277614 & 3.926178 & 3.648563 \\
19 & 0.000000 & 3.648563 & 0.000000 & 3.648563 & 3.648563 \\
29 & 0.000000 & 3.648563 & 0.000000 & 3.648563 & 3.648563 \\
\bottomrule
\end{tabular}
\end{table}

Although Greedy Coverage optimizes an average coverage objective, it matches or improves the worst case deficit $\max_{t \in T}\Delta_S(t)$ at several budgets shown in Table~\ref{tab:gc_mg_certificate}.

The empirical estimate $\widehat{\varepsilon}(T)$ makes the discretization term in Theorem 3.1 computable. The reported decomposition shows atom selection controls the alignment deficit at small budgets, while net refinement controls the certificate once $\max_{t \in T}\Delta_S(t)$ is small. This provides a diagnostic for deciding whether to increase $B$ or to enrich the direction set $T$.

\subsection{Frozen feature CIFAR classification}
\label{app:frozen-details}

A pretrained ResNet 18 backbone is used to extract frozen features. Tasks are binary, defined by two classes from CIFAR 10 or CIFAR 100, with labels mapped to $\{-1,+1\}$. The training features are split into disjoint parts with proportions
\[
[0.5,\,0.2,\,0.15,\,0.15],
\]
used for classifier training, dictionary mining, selection directions, and reporting directions.

A linear classifier is trained with Adam and fixed hyperparameters. A dictionary is mined in feature space using a projected gradient attack in $\ell_2$ with radius $\varepsilon$. Normalized perturbation directions are stored, and certified accuracy $\mathrm{CertAcc}(\varepsilon)$ alongside projected gradient robust accuracy are evaluated.

\paragraph{Certified accuracy.}
\label{app:certacc}
Let $w\in\mathbb{R}^d$ be the weight vector of the linear classifier, and let
$\mathcal D=\{d_i\}_{i=1}^N \subseteq \mathbb{R}^d$ be the directional dictionary in feature space.
Define
\[
\mathrm{pen}_{\mathcal D}(w)\ :=\ \max_{i\in[N]} \big|\langle d_i, w\rangle\big|.
\]
For $\varepsilon>0$, define
\[
\mathrm{CertAcc}_{\mathcal D}(\varepsilon)\ :=\ \frac{1}{n}\sum_{j=1}^n
\mathbf{1}\!\left\{ y_j\langle w, x_j\rangle > \varepsilon\,\mathrm{pen}_{\mathcal D}(w)\right\},
\]
where $(x_j,y_j)$ are test examples in feature space and $y_j\in\{-1,+1\}$.
For a selected subset $S\subseteq[N]$, the reported metric is $\mathrm{CertAcc}_{S}(\varepsilon)$, obtained by replacing
$\mathrm{pen}_{\mathcal D}$ with $\mathrm{pen}_{S}(w)=\max_{i\in S}|\langle d_i,w\rangle|$.





\subsection{Full Instance Context for Budget Efficiency and Reliability Analysis}
\label{app:all100_sec53}

Section~\ref{sec:synthetic_results} reports onset budget $B^\star$ and budget-wise success rates on the subset of instances that are robust feasible under the full dictionary baseline, with zero violations. This conditioning isolates budget efficiency from baseline infeasibility.

Across all generated seeds ($N=100$), the full dictionary baseline is robust feasible on 64 seeds (64.0\%) and infeasible on 36 seeds (36.0\%). Table~\ref{tab:all100_context_sec53} summarizes policy-level certification counts over all 100 seeds on the same budget grid used in Section~\ref{sec:synthetic_results}, making the screening step explicit.

\begin{table}[!h]
\centering\small
\setlength{\tabcolsep}{4pt}
\caption{Policy-level certification summary over all generated instances ($N=100$) for the experiment in Section~\ref{sec:synthetic_results}. Full feasible and Full infeasible refer to feasibility under the full dictionary baseline.}
\label{tab:all100_context_sec53}
\begin{tabular}{lcccc}
\toprule
Policy & Full feasible & Full infeasible & Total certified & Not certified \\
\midrule
greedy-maxgap   & 58/64 (90.6\%) & 3/36 (8.3\%) & 61/100 (61.0\%) & 39/100 (39.0\%) \\
greedy-coverage & 55/64 (85.9\%) & 3/36 (8.3\%) & 58/100 (58.0\%) & 42/100 (42.0\%) \\
max-gamma       & 53/64 (82.8\%) & 3/36 (8.3\%) & 56/100 (56.0\%) & 44/100 (44.0\%) \\
\bottomrule
\end{tabular}
\end{table}

The $N=64$ subset in Section~\ref{sec:synthetic_results} is the relevant population for onset budget and budget-wise efficiency comparisons because those metrics are meaningful only when the full dictionary baseline is robust feasible. Table~\ref{tab:all100_context_sec53} complements that analysis by reporting certification outcomes on the full instance pool.

Certification on full-baseline infeasible seeds does not imply feasibility of the full dictionary baseline. It means that a policy achieves zero violations for that seed at some evaluated budget in its policy-specific sweep, while full-baseline feasibility is defined with respect to the full dictionary baseline.

%% file: content/0_appendix_out.tex
\subsection{Out-of-sample calibration usage and conservatism}
\label{app:oos-usage}
\label{app:calibration_ablation_f1}

Radius calibration applies to the synthetic geometric validation experiment. For a selected subset $S$, a radius $\widehat r(S)$ is estimated from a calibration set, reporting the guarantee of Theorem~\ref{thm:oos-feas}. In the frozen feature classification experiment, radius is not calibrated; the reported metric is the certified accuracy induced by the atomic model and the dictionary.

\paragraph{Calibration conservatism and out-of-sample feasibility.}
This subsection reports a calibration ablation for the synthetic pipeline. The analysis compares DKW union calibration against split calibration on a held out test split, detailing sensitivity to the target violation level $\alpha$ and confidence parameter $\delta$.

Results are computed from \texttt{calibration\_ablation.csv}, storing one row per $(B,\alpha,\delta,\text{method})$ configuration. Configurations include budgets $B \in \{9,19,29\}$, target levels $\alpha \in \{0.01,0.03,0.05,0.1\}$, confidence levels $\delta \in \{10^{-5},10^{-4},10^{-3}\}$, and calibration methods \{DKW union, Split\}.

Table~\ref{tab:f1_calibration_method_delta} summarizes the ablation by calibration method and confidence parameter. DKW union calibration produces higher correction terms and calibrated radii, with negative mean violation gaps $(\widehat{\mathrm{viol}}-\alpha)$ across tested $\delta$. Split calibration yields lower radii and positive mean violation gaps.

For DKW union, the mean correction term changes from 0.2316 at $\delta=10^{-3}$ to 0.2367 at $\delta=10^{-5}$, and the mean violation gap remains negative (-0.0082 to -0.0086). For split calibration, the mean correction term changes from 0.0616 to 0.0781, while the mean violation gap remains positive (0.0048 to 0.0037).

\begin{table}[h]
\centering
\caption{Calibration ablation summary by method and confidence parameter. Gap is observed violation minus target level.}
\label{tab:f1_calibration_method_delta}
\scriptsize
\begin{tabular}{l c c c c c c c}
\hline
Method & $\delta$ & Mean corr. & Mean radius & Mean viol. & Mean gap & Max gap & Min gap \\
\hline
DKW union & $10^{-5}$ & 0.2367 & 4.4114 & 0.0389 & -0.0086 & 0.0000 & -0.0254 \\
DKW union & $10^{-4}$ & 0.2342 & 4.4089 & 0.0391 & -0.0084 & 0.0000 & -0.0250 \\
DKW union & $10^{-3}$ & 0.2316 & 4.4064 & 0.0393 & -0.0082 & 0.0000 & -0.0248 \\
Split & $10^{-5}$ & 0.0781 & 4.2529 & 0.0512 & 0.0037 & 0.0088 & 0.0008 \\
Split & $10^{-4}$ & 0.0704 & 4.2451 & 0.0519 & 0.0044 & 0.0100 & 0.0016 \\
Split & $10^{-3}$ & 0.0616 & 4.2364 & 0.0524 & 0.0048 & 0.0106 & 0.0026 \\
\hline
\end{tabular}
\end{table}

Table~\ref{tab:f1_calibration_budget_method} summarizes these quantities by budget after averaging over tested $(\alpha,\delta)$ pairs. DKW union radius increases with budget because the union correction grows with combinatorial family size, while split correction remains constant.

\begin{table}[h]
\centering
\caption{Calibration ablation summary by budget and calibration method, averaged over $(\alpha,\delta)$ settings.}
\label{tab:f1_calibration_budget_method}
\small
\begin{tabular}{c l c c c c}
\hline
$B$ & Method & Mean corr. & Mean radius & Mean viol. & Mean gap \\
\hline
9 & DKW union & 0.1809 & 4.3557 & 0.0429 & -0.0046 \\
9 & Split & 0.0700 & 4.2448 & 0.0518 & 0.0043 \\
19 & DKW union & 0.2397 & 4.4144 & 0.0385 & -0.0090 \\
19 & Split & 0.0700 & 4.2448 & 0.0518 & 0.0043 \\
29 & DKW union & 0.2818 & 4.4566 & 0.0360 & -0.0115 \\
29 & Split & 0.0700 & 4.2448 & 0.0518 & 0.0043 \\
\hline
\end{tabular}
\end{table}

Figure~\ref{fig:f1_calibration_conservatism} reports correction magnitude at the default confidence level, and Figure~\ref{fig:f1_calibration_alpha_delta} reports the ablation over $(\alpha,\delta)$ for both calibration methods. The left panel plots observed out-of-sample violation against the target level at default confidence. The right panel plots the violation gap as a function of the confidence parameter at fixed $\alpha=0.05$.

\begin{figure}[h]
    \centering
    \includegraphics[width=0.6\linewidth]{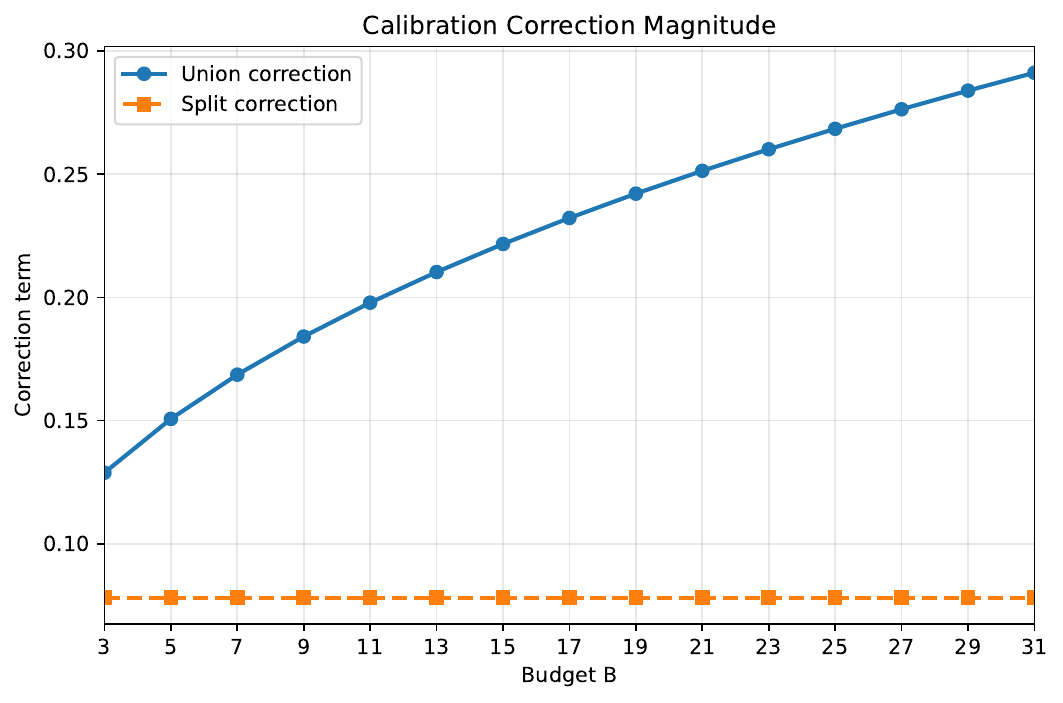}
    \caption{Calibration correction magnitude at the default confidence level. DKW union correction increases with budget; split correction is constant.}
    \label{fig:f1_calibration_conservatism}
\end{figure}

\begin{figure}[h]
    \centering
    \includegraphics[width=0.98\linewidth]{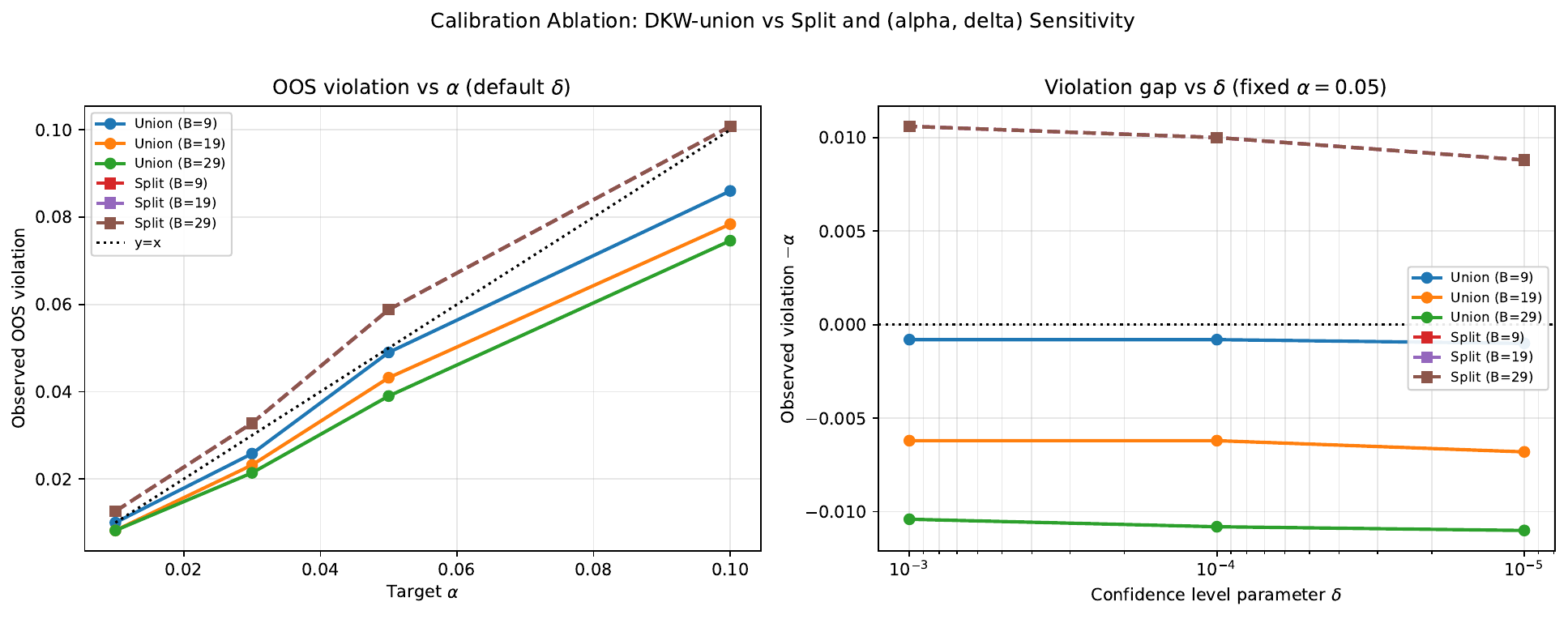}
    \caption{Calibration ablation with $(\alpha,\delta)$ sweeps. Left: observed out-of-sample violation versus target $\alpha$ at default confidence level. Right: violation gap $(\widehat{\mathrm{viol}}-\alpha)$ versus $\delta$ at fixed $\alpha=0.05$.}
    \label{fig:f1_calibration_alpha_delta}
\end{figure}

Table~\ref{tab:f1_calibration_representative} reports configurations at $\alpha=0.05$ for tested budgets and confidence levels. Metrics include raw calibration quantile, correction term, calibrated radius, observed violation, and violation gap.

\begin{table}[h]
\centering
\caption{Calibration settings at fixed $\alpha=0.05$ for tested budgets and confidence levels.}
\label{tab:f1_calibration_representative}
\scriptsize
\begin{tabular}{c c l c c c c c}
\hline
$B$ & $\delta$ & Method & Quantile & Corr. & Radius & Viol. & Gap \\
\hline
9 & $10^{-5}$ & DKW union & 3.9942 & 0.1841 & 4.1783 & 0.0490 & -0.0010 \\
9 & $10^{-5}$ & Split & 3.9942 & 0.0781 & 4.0723 & 0.0588 & 0.0088 \\
9 & $10^{-4}$ & DKW union & 3.9942 & 0.1810 & 4.1752 & 0.0492 & -0.0008 \\
9 & $10^{-4}$ & Split & 3.9942 & 0.0704 & 4.0646 & 0.0600 & 0.0100 \\
9 & $10^{-3}$ & DKW union & 3.9942 & 0.1778 & 4.1720 & 0.0492 & -0.0008 \\
9 & $10^{-3}$ & Split & 3.9942 & 0.0616 & 4.0559 & 0.0606 & 0.0106 \\
19 & $10^{-5}$ & DKW union & 3.9942 & 0.2421 & 4.2363 & 0.0432 & -0.0068 \\
19 & $10^{-5}$ & Split & 3.9942 & 0.0781 & 4.0723 & 0.0588 & 0.0088 \\
19 & $10^{-4}$ & DKW union & 3.9942 & 0.2397 & 4.2339 & 0.0438 & -0.0062 \\
19 & $10^{-4}$ & Split & 3.9942 & 0.0704 & 4.0646 & 0.0600 & 0.0100 \\
19 & $10^{-3}$ & DKW union & 3.9942 & 0.2373 & 4.2315 & 0.0438 & -0.0062 \\
19 & $10^{-3}$ & Split & 3.9942 & 0.0616 & 4.0559 & 0.0606 & 0.0106 \\
29 & $10^{-5}$ & DKW union & 3.9942 & 0.2839 & 4.2781 & 0.0390 & -0.0110 \\
29 & $10^{-5}$ & Split & 3.9942 & 0.0781 & 4.0723 & 0.0588 & 0.0088 \\
29 & $10^{-4}$ & DKW union & 3.9942 & 0.2819 & 4.2761 & 0.0392 & -0.0108 \\
29 & $10^{-4}$ & Split & 3.9942 & 0.0704 & 4.0646 & 0.0600 & 0.0100 \\
29 & $10^{-3}$ & DKW union & 3.9942 & 0.2798 & 4.2740 & 0.0396 & -0.0104 \\
29 & $10^{-3}$ & Split & 3.9942 & 0.0616 & 4.0559 & 0.0606 & 0.0106 \\
\hline
\end{tabular}
\end{table}

The pairwise correction ratio between DKW union and split calibration, computed at matched $(B,\alpha,\delta)$ settings, has mean 3.3715 and maximum 4.5388. The radius increase induced by DKW union averages 0.1641, while the change in violation gap $(\widehat{\mathrm{viol}}-\alpha)$ averages -0.0127.

DKW union calibration provides feasibility control with higher correction terms and calibrated radii, while split calibration yields lower radii but can exceed the target violation level on the held out test split.